\documentclass[a4paper]{article}
\pdfoutput=1
\usepackage{graphicx}
\usepackage{listings}
\usepackage{xcolor}
\usepackage{amstext}
\usepackage{custom}
\usepackage{geometry}
\usepackage{multirow}
\usepackage[all]{xy}
\usepackage{wrapfig}
\usepackage{fullpage}

\usepackage{algorithm}
\usepackage{algpseudocode}

\usepackage[utf8]{inputenc} 
\usepackage[T1]{fontenc}    
\usepackage{hyperref}       
\usepackage{url}            
\usepackage{booktabs}       
\usepackage{amsfonts}       
\usepackage{nicefrac}       
\usepackage{microtype}      
\usepackage{lipsum}
\usepackage{fancyhdr}       
\usepackage{graphicx}       
\graphicspath{{media/}}     
\usepackage{amsfonts}
\usepackage{amsmath}
\usepackage{amssymb,bbm,color,xcolor,bm}
\usepackage{graphicx} 
\usepackage{amsthm}
\usepackage{appendix}

\newcommand{\Prob}{\mathbb{P}}

\usepackage{tabularx}
\usepackage{ragged2e} 

\newtheorem{theorem}{Theorem}
\newtheorem{definition}{Definition}
\newtheorem{example}{Example}

\newtheorem{remark}{Remark}
\usepackage{tikz}

\usepackage[round]{natbib}

\title{Fair Risk Control: A Generalized Framework for Calibrating Multi-group Fairness Risks}
\author{
Lujing Zhang\footnotemark[1]
\and
Aaron Roth\footnotemark[2]
\and
Linjun Zhang\footnotemark[3], \footnotemark[4]
}

\begin{document}
\maketitle
\footnotetext[1]{Work was done during Lujing Zhang's remote research internship at Rutgers and Penn. Email: misdrifter@stu.pku.edu.cn}
\footnotetext[2]{University of Pennsylvania. Email: aaroth@cis.upenn.edu}
\footnotetext[3]{Rutgers University. Email: linjun.zhang@rutgers.edu}
\footnotetext[3]{Corresponding Author.}

\begin{abstract}
This paper introduces a framework for post-processing machine learning models so that their predictions satisfy multi-group fairness guarantees. Based on the celebrated notion of multicalibration, we introduce $(\bm s,\mathcal{G}, \alpha)-$GMC (Generalized Multi-Dimensional Multicalibration) for multi-dimensional mappings $\bm s$, constraint set $\mathcal{G}$, and a pre-specified threshold level $\alpha$. We propose associated algorithms to achieve this notion in general settings. This framework is then applied to diverse scenarios encompassing different fairness concerns, including false negative rate control in image segmentation, prediction set conditional uncertainty quantification in hierarchical classification, and de-biased text generation in language models. We conduct numerical studies on several datasets and tasks.
\end{abstract}

\section{Introduction}
A common theme across the fairness in machine learning literature is that some measure of \emph{error} or \emph{risk} should be equalized across sub-populations. Common measures evaluated across demographic groups include false positive and false negative rates \citep{hardt2016equality} and calibration error \citep{kleinberg2016inherent,chouldechova2017fair}. Initial work in this line gave methods for equalizing different risk measures on disjoint groups. A second generation of work gave methods for equalizing measures of risk across groups even when the groups could intersect -- e.g. for false positive and negative rates \citep{kearns2018preventing}, calibration error \citep{hebertjohnson2018calibration}, regret \citep{blum2019advancing,rothblum2021multi}, prediction set coverage \citep{jung2021moment,jung2022batch, deng2023happymap}, among other risk measures. In general, distinct algorithms are derived for each of these settings, and they are generally limited to one-dimensional predictors of various sorts.

In this work, we propose a unifying framework for fair risk control in settings with multi-dimensional outputs, based on multicalibration \citep{hebertjohnson2018calibration}. This framework is developed as an extension of the work by \cite{deng2023happymap,NR23}, and addresses the need for calibrating multi-dimensional output functions. To illustrate the usefulness of this framework, we apply it to a variety of settings, including false negative rate control in image segmentation, prediction set conditional coverage guarantees in hierarchical classification, and de-biased text generation in language models. These applications  make use of the additional power granted by our multi-dimensional extension of multicalibration. 

\subsection{Related Work}
Multicalibration was introduced by \cite{hebertjohnson2018calibration} as a fairness motivated constraint that informally asks that a 1-dimensional predictor of a binary-valued outcome be unbiased, conditional on both its own prediction and on membership of the input in some number of pre-defined groups (see also a line of prior work that asks for a similar set of guarantees under slightly different conditions \citep{dawid1985calibration,sandroni2003calibration,foster2006calibration}). Subsequently, multicalibration has been generalized in a number of ways. \cite{jung2021moment} generalizes multicalibration to real-valued outcomes, and defines and studies a variant of multicalibration that predicts variance and higher moments rather than means. \cite{gupta2022online} extends the study of multicalibration of both means and moments to the online setting, and defines a variant of mulicalibration for quantiles, with applications to uncertainty estimation. \cite{bastani2022practical,jung2022batch} gives more practical variants of quantile multicalibration with applications to conditional coverage guarantees in conformal prediction, together with experimental evaluation. \cite{deng2023happymap} gives an abstract generalization of 1-dimensional multicalibration, and show how to cast other algorithmic fairness desiderata like false positive rate control in this framework. \cite{NR23} gives a characterization of the scope of 1-dimensional multicalibration variants via a connection to property elicitation: informally, a property of a distribution can be multicalibrated if and only if it minimizes some 1-dimensional separable regression function. The primary point of departure of this paper is that we propose a multi-dimensional generalization of multicalibration: it can be viewed as the natural multi-dimensional generalization of \cite{deng2023happymap}. 

Another line of work generalizes multicalibration in an orthogonal direction, leaving the outcomes binary valued but generalizing the class of checking rules that are applied. \cite{dwork2021outcome} defines outcome indistinguishability, which generalizes multicalibration to require indistinguishability between the predicted and true label distributions with respect to a fixed but arbitrary set of distinguishers. \cite{kakade2008deterministic,foster2018smooth} define ``smooth calibration'' that relaxes calibration's conditioning event to be a smooth function of the prediction. \cite{gopalan2022low} defines a hierarchy of relaxations called low-degree multicalibration that further relaxes smooth calibration and demonstrates desirable statistical properties. \cite{zhao2021calibrating} and \cite{noarov2023high} define notions of calibration tailored to the objective function of a downstream decision maker. These last lines of work focus on multi-dimensional outputs.

These lines of work are part of a more general literature studying \emph{multi-group fairness}. Work in this line aims e.g. to minimize disparities between false positive or false negative rates across groups \citep{kearns2018preventing,kearns2019empirical}, or to minimize regret (measured in terms of accuracy) simultaneously across all groups \citep{blum2019advancing,rothblum2021multi,globus2022algorithmic,tosh2022simple}. A common theme across these works is that the groups may be arbitrary and intersecting.

\subsection{Notation}
Let $\mathcal{X}$ represent a feature domain, $\mathcal{Y}$ represent a label domain, and $\mathcal{D}$ denote a joint (feature, label)  data distribution. For a finite set $A$, we use $|A|$ and $\Delta A$, to denote the cardinality of $A$ and the simplex over $A$ respectively. Specifically, $\Delta A = \{(p_1, p_2, \ldots, p_{|A|}): 0 \leq p_i \leq 1, \sum_{i=1}^{|A|} p_i = 1\}$. Given a set $\mathcal{F}$, we use ${\rm Proj}_{\mathcal{F}}$ to denote the $\ell_2$-projection onto the set.
We also introduce some shorthand notation. For two vectors $\bm{a}$ and $\bm{b}$, $\langle\bm a,\bm b\rangle$ represents their inner product. For a positive integer $T$, we define $[T] = \{1, 2, \ldots, T\}$. For a function $\bm f(\bm x)=(f_1(\bm x), f_2(\bm x),...,f_m(\bm x))$, we denote $\lVert \bm f \rVert_{\infty} = \sup_{\bm x\in \mathcal{X}, i\in [m]}[f_i(\bm x)]$.

\section{Formulation and Algorithm}
\subsection{A generalized notion of Multicalibration} 
Let $\bm x\in \mathcal{X}$ represent the feature vector of the input, $\bm y\in \mathcal{Y}$ represent the label, and let $\bm h(\bm x)\in \mathcal{H}$ denote a multi-dimensional scoring function associated with the input. For example, in image segmentation tasks, $\bm h(\bm x)\in \mathbb{R}^k$ ($k$ is the number of pixels) is intended to approximate the probability of a pixel being part of a relevant segment, often learned by a neural network. In text generation tasks, $\bm h(\bm x)$ is the distribution over the vocabulary produced by a language model given context $\bm x$. 

For $\bm x\in \mathcal{X}$, consider an output function $\bm f:\mathcal{X}\to \mathcal{F} \subset \mathbbm{R}^m$, defined as $\bm f(\bm x) = (f_1(\bm x), \ldots, f_m(\bm x))$, where $\mathcal{F}$ is a convex set. We denote the class of functions that $\bm f$ belongs to by $\mathcal Q$. 
 For example, in text generation tasks, $\bm f(\bm x)$ is the calibrated distribution over the output vocabulary and is multi-dimensional (with dimension equal to the vocabulary size); in binary classification tasks where $h$ and $f$ are both scalars, $f(\bm x)$ is the threshold used to convert the raw score $h(\bm x)$ into binary predictions, i.e.  $\mathbbm{1}_{\{h(\bm x)>f(\bm x)\}}$. 
 
We write $\bm s(\bm f, \bm x, \bm h, \bm y, \mathcal{D}): \mathcal Q  \times \mathcal{X}\times \mathcal{H} \times \mathcal{Y} \times \mathcal{P} \to \mathbb{R}^l$ to denote a mapping functional of interest, where $\mathcal{D}$ is the joint distribution of $(\bm x, \bm h, \bm y)$ and $\mathcal{P}$ is the distribution space. Here, $\bm s$ is set to be a functional of $\bm f$ rather than a function of $\bm f(\bm x)$, which offers us more flexibility that will be useful in our applications. For example, in text generation, where $\bm h(\bm x)\in \Delta\mathcal{Y}$ is the distribution over tokens output by an initial language model, our goal might be to find $\bm f(\bm x)\in \Delta\mathcal{Y}$, an adjusted distribution over tokens $y\in \mathcal{Y}$ with $|\mathcal{Y}|=m$. In this case we could set $\bm s=\bm f(\bm x) -\mathbb{E}_{\bm x}\bm f(\bm x)\in \mathbbm{R}^m$ to be the mapping functional. 
We can calibrate the probabilities (through $\bm s$) to be ``fair'' in some way -- e.g. that the probability of outputting various words denoting professions should be the same regardless of the gender of pronouns used in the prompt. We note that we do not always use the dependence of $\bm s$ on all of its inputs and assign different $\bm s$ in different settings.

We write $\mathcal{G}$ to denote the class of functions that encode demographic subgroups (along with other information) and for each $\bm g\in \mathcal{G}$, $\bm g(\bm f(\bm x), \bm x)\in \mathbb{R}^l$, consistent with the dimension of $\bm s(\bm f, \bm x, \bm h, \bm y, \mathcal{D})$ so that we can calibrate over every dimension of $\bm s$. For example, when $l=1$, $\mathcal{G}$ can be set to be the indicator function of different sensitive subgroups of $\mathcal{X}$. Alternately, in fair text generation tasks, when the dimension of $\bm s$ equals the size of the set $\mathcal{Y}$, denoted as $l=m$, we can set the vector $\bm g \in \mathcal{G}$ to have a value of $1$ in the dimensions corresponding to certain types of sensitive words, and $0$ in all other dimensions. 

We now formally introduce the $(\bm s,\mathcal{G}, \alpha)$-Generalized Multicalibration ($(\bm s,\mathcal{G}, \alpha)$-GMC) definition.
\begin{definition}[$(\bm s,\mathcal{G}, \alpha)$-GMC]
    Let $\bm x, \bm h,\bm y, \mathcal{D}$ denote the feature vector, the scoring function, the label vector, and the joint distribution of $(\bm x, \bm h, \bm y)$ respectively. Given a function class $\mathcal{G}$, mapping functional $\bm s$, and a threshold $\alpha>0$, we say $\bm f$ satisfies $(\bm s,\mathcal{G}, \alpha)$-Generalized Multicalibration ($(\bm s,\mathcal{G}, \alpha)$-GMC) if 
        $$\mathbb{E}_{(\bm x, \bm h, \bm y)\sim \mathcal{D}}[\langle  \bm s(\bm f, \bm x,\bm h,\bm y, \mathcal{D}),\bm g(\bm f(\bm x),\bm x)\rangle]\leq \alpha,\quad \forall \bm g\in \mathcal{G}.$$
\end{definition}

$(\bm s,\mathcal{G}, \alpha)$-GMC is a flexible framework that can instantiate many existing multi-group fairness notions, including $s$-HappyMap \citep{deng2023happymap}, property multicalibration \citep{NR23}, calibrated multivalid coverage \citep{jung2022batch} and outcome indistinguishability \citep{dwork2021outcome}. More generally, compared to these notions, $(\bm s,\mathcal{G}, \alpha)$-GMC extends the literature in two ways. First, it allows the functions $\bm s$ and $\bm g$ to be multi-dimensional (most prior definitions look similar, but with $1$-dimensional $\bm s$ and $\bm g$ functions). Second, the function $\bm s$ here is more general and allowed to be a \emph{functional} of $\bm f$ (rather than just a function of $\bm f(\bm x)$, the evaluation of $\bm f$ at $\bm x$). These generalizations will be important in our applications.

\subsection{Algorithm and  Convergence Results}
To achieve $(\bm s,\mathcal{G}, \alpha)$-GMC, we present the $(\bm s,\mathcal{G}, \alpha)$-GMC Algorithm, which can be seen as  a natural generalization of algorithms used for more specific notions of multicalibration in previous work \citep{hebertjohnson2018calibration,dwork2021outcome,jung2022batch,deng2023happymap}: 

\begin{algorithm}[H]
\caption{$(\bm s,\mathcal{G}, \alpha)$-GMC lgorithm}\label{alg:general}
\begin{algorithmic}
\State{\bf Input:} step size $\eta\ >0$, initialization $\bm f^{(0)} \in \mathcal Q$, max iteration $T$. {\bf Initialization:} $t=0$.
\While {$t< T, \exists \bm{g}^{(t)}\in \mathcal{G} \; s.t: 
\mathbb{E}_{(\bm x, \bm h, \bm y)\sim \mathcal{D}}[\langle \bm{s}(\bm f^{(t)}, \bm x, \bm h, \bm y, \mathcal{D}),\bm{g}^{(t)}(\bm f^{(t)}(\bm x), \bm x)\rangle]>\alpha$}
\State{Let $\bm{g}^{(t)}\in \mathcal{G}$ be an arbitrary function satisfying the condition in the while statement}
\State{$\bm f^{(t+1)}(\bm x)={\rm Proj}_{\mathcal{F}}\left(\bm f^{(t)}(\bm x)-\eta\bm{g}^{(t)}(\bm f^{(t)}(\bm x), \bm x)\right)$} 
\State{$t=t+1$}
\EndWhile
\State{\bf Output:} $\bm f^{(t)}$
\end{algorithmic}
\end{algorithm}

It is worth noting that our goal involves functionals concerning our objective function $\bm {f}$ in order to capture its global properties. We aim to find a function $\bm{f}$ such that a functional associated with it (obtained by taking the expectation over $\bm{x}$) satisfies the inequalities we have set to meet different fairness demands. Before delving into the main part of our convergence analysis, we introduce some definitions related to functionals. Examples of these definitions can be found in the appendix Section~\ref{sec:eg}. 
\begin{definition}[The derivative of a functional]\label{derivative}
    Given a function $\bm f:\mathcal{X}\to \mathcal{F}$, consider a functional $\mathcal{L}(\bm f, \mathcal{D}):\mathcal Q \times \mathcal{P} \to \mathbb{R}$, where $\mathcal Q$ is the function space of $\bm f$ and $\mathcal{P}$ is a distribution space over $\mathcal{X}$. Assume that $\mathcal{L}(\bm f, \mathcal{D})$ follows the formulation that $\mathcal{L}(\bm f, \mathcal{D})=\mathbb{E}_{\bm x\sim \mathcal{D}}[L(\bm f(\bm x))]$. The derivative function of $\mathcal{L}(\bm f, \mathcal{D})$ with respect to $\bm f$, denoted as $\nabla_{\bm f}\mathcal{L}(\bm f, \mathcal{D}):\mathcal{X} \to \mathcal{F}$, exists if 
    $\forall \bm w\in \mathcal Q, \bm y\in\mathbb{R}^m, \mathcal D\in\mathcal P,
        \mathbb{E}_{\bm x \sim \mathcal{D}}[\langle \nabla_{\bm f}\mathcal{L}(\bm f, \mathcal{D}), \bm w\rangle]
        =\frac{\partial}{\partial \epsilon}\left. \mathcal{L}(\bm f+\epsilon \bm w, \mathcal{D})\right|_{\epsilon=0}.$

In the following, we introduce the definitions of convexity and smoothness of a functional. 

\end{definition}

\begin{definition}[Convexity of a functional]
    Let $\mathcal{L}$ and $\bm f$ be defined as in  Definition \ref{derivative}. A functional $\mathcal{L}$ is convex with respect to $\bm f$ if for any $\bm{f_1}, \bm{f_2}\in \mathcal Q,\mathcal{L}(\bm {f_1}, \mathcal{D})-\mathcal{L}(\bm {f_2}, \mathcal{D})\geq \mathbb{E}_{\bm x \sim \mathcal{D}}[\langle \nabla_{\bm f} \mathcal{L}(\bm {f_2},\mathcal{D}), \bm {f_1}-\bm{f_2} \rangle]. $
\end{definition}

\begin{definition}[$K_{\mathcal{L}}$-smoothness of a functional]
    Let $\mathcal{L}$ and $\bm f$ be defined as in Definition  \ref{derivative}. A functional $\mathcal{L}$ is $K_{\mathcal{L}}-$smooth if for any $\bm{f_1}, \bm{f_2}\in \mathcal Q,$
    $
        \mathcal{L}(\bm {f_1}, \mathcal{D})-\mathcal{L}(\bm {f_2}, \mathcal{D})\leq \mathbb{E}_{\bm x\sim \mathcal{D}}[\langle \nabla \mathcal{L}(\bm {f_2}, \mathcal{D}), \bm {f_1} -\bm{f_2}\rangle]
        +\mathbb{E}_{\bm x\sim \mathcal{D}}[\frac{K_{\mathcal{L}}}{2}\lVert\bm{f_1}-\bm{f_2}\rVert^2].\nonumber
    $

\end{definition}

We will prove that this algorithm converges and outputs an $\bm f$ satisfying $(\bm s,\mathcal{G}, \alpha)$-GMC whenever the following assumptions are satisfied. These are multidimensional generalizations of the conditions given by \cite{deng2023happymap}.

\textbf{Assumptions}
\begin{enumerate}
    \item[(1).] There exists a potential functional $\mathcal{L}(\bm f, \bm h,\bm y, \mathcal{D})$, such that $\nabla_{\bm f}\mathcal{L}(\bm f, \bm h, \bm y, \mathcal{D})(\bm x) = \bm{s}(\bm f, \bm x, \bm h, \bm y, \mathcal{D}), $\\$\text{ and }\mathcal{L}(\bm f, \bm h,\bm y, \mathcal{D}) \text{ is }K_{\mathcal{L}}$-smooth with respect to $\bm f$ for any $\bm x\in \mathcal{X}.$
    \item[(2).] Let $\bm f^*(\bm x)\triangleq {\rm Proj}_{\mathcal{F}}\bm f(\bm x)$ for all $\bm x\in \mathcal{X}$. For any $\bm f\in \mathcal{Q}$, $\mathcal{L}(\bm f^*, \bm h, \bm y, \mathcal{D}) \leq \mathcal{L}(\bm f, \bm  h, \bm y, \mathcal{D})$ .
    \item[(3).] There exists a positive number $B$, such that for all $\bm g\in \mathcal{G}$ and all $\bm f \in \mathcal{Q}$, $ \mathbb{E}_{\bm x\sim \mathcal{D}}[\lVert \bm{g}(\bm f(\bm x), \bm x)\rVert^2 ]\leq B.$
    \item[(4).] There exists two numbers $C_l,C_u$ such that for all  $\bm f\in \mathcal Q$, 
    $\quad \mathcal{L}(\bm f, \bm h, \bm y, \mathcal{D})\geq C_l,$
    $ \mathcal{L}(\bm f^{(0)}, \bm h, \bm y, \mathcal{D})\leq C_u.$
\end{enumerate}

Assumption (1) says that a potential functional $\mathcal{L}$ exists and it satisfies a $K_{\mathcal{L}}$-smoothness condition with respect to $\bm f$. For example, when $\bm f$ is a predicted distribution, we often set $\bm s=\bm f(\bm x)-\mathbb{E}_{\bm x \sim \mathcal{D}}\bm f(\bm x)$. In this situation, $\mathcal{L}=\mathbb{E}_{\bm x \sim \mathcal{D}}[\frac{1}{2}\lVert \bm f(\bm x)-\mathbb{E}_{\bm x\sim\mathcal{D}}\bm f(\bm x)\rVert^2]$ satisfies the assumption.

Assumption (2) states that the potential function decreases when projected with respect to $\bm f$. A specific example is when $\mathcal{F}=\mathcal{Y}=[0,1]$ and $\mathcal{L}=\mathbb{E}_{(\bm x, y)\sim \mathcal{D}}|f(\bm x) - y|^2.$

Assumption (3) states that the $\ell_2$-norm of the functions in $\mathcal{G}$ is uniformly bounded. It always holds when $\mathcal{G}$  contains indicator functions, which is the most common case in fairness-motivated problems (these are usually the indicator functions for subgroups of the data).

Assumption (4) says that the potential functional $\mathcal{L}$ is lower bounded and this generally holds true when $\mathcal{L}$ is convex. One concrete example is when $s(f(\bm x),h,y)=f(\bm x)-y$ and we have $\mathcal{L}(f, h, y, \mathcal{D})=\mathbb{E}_{\bm x\sim \mathcal{D}}[(f(\bm x)-y)^2]$, which is lower bounded by $0$.

\begin{theorem}
Under Assumptions 1-4, the $(\bm s,\mathcal{G}, \alpha)$-GMC Algorithm with a suitably chosen $\eta = \mathcal{O}(\alpha/(K_{\mathcal{L}}B))$ converges in $T=\mathcal{O}(\frac{2K_{\mathcal{L}}(C_u-C_l)B)}{\alpha^2})$ iterations and outputs a function $\bm f$ satisfying
$$\mathbb{E}_{(\bm x,\bm h,\bm y)\sim \mathcal{D}}[\langle \bm s(\bm f, \bm x, \bm h, \bm y, \mathcal{D}),\bm{g}(\bm f(\bm x), \bm x)\rangle]\leq \alpha, \forall \bm{g} \in \mathcal{G}.$$
\end{theorem}

The  proof is provided in Appendix C. At a high level, if we consider $\bm g$ as a generalized direction vector and $\bm s$ as the gradient of $\mathcal{L}$, each violation can be interpreted as detecting a direction where the first-order difference of $\mathcal{L}$ is significant. By introducing the assumption of smoothness, our update can result in a decrease in $\mathcal{L}$ that exceeds a constant value. Since $\mathcal{L}$ is lower bounded by assumption, the updates can terminate as described.

\subsection{Finite-Sample Results}
We have presented Algorithm \ref{alg:general} as if we have direct access to the true data distribution $\mathcal{D}$. In practice, we only have a finite calibration set $D$, whose data is sampled $i.i.d$ from $\mathcal{D}$. In this subsection, we show how a  variant of Algorithm \ref{alg:general} achieves the same goal from finite samples.

First, we introduce a useful measure which we call the \textit{dimension of the function class}, as similarly defined by \cite{kim2019multiaccuracy, deng2023happymap}. For a dataset $D$, we use $\mathbb{E}_{(\bm x, \bm h, \bm y)\sim D}$ to denote the empirical expectation over $D$. 
We need $T$ datasets in all and we assume that the whole sample size is $m$ (${m}/{T}$ for each dataset).
\begin{definition}[Dimension of the function class]\label{dimension}
    We use $d(\mathcal{G})$ to denote the dimension of class $\mathcal{G}$, defined to be a quantity such that if the sample size $m \geq C_1 \frac{d(\mathcal{G})+\log (1 / \delta)}{\alpha^2}$, then a random sample $S_m$ of $m$ elements from $\mathcal{D}$ guarantees uniform convergence over $\mathcal{G}$ with error at most $\alpha$ with failure probability at most $\delta$. That is, for any fixed $\bm f$ and fixed $\bm s$ with $\|\bm s\|_{\infty} \leq C_2$ ($C_1,C_2>0$ are universal constants):
        $$\sup _{\bm g \in \mathcal{G}}|\mathbb{E}_{(\bm x, \bm h, \bm y) \sim \mathcal{D}}[\langle\bm s(\bm f, \bm x, \bm h, \bm y, \mathcal{D}), \bm g(\bm f(\bm x), \bm x)\rangle] \nonumber\\
        -\mathbb{E}_{(\bm x, \bm h, \bm y) \sim S_m}[\langle\bm s(\bm f, \bm x, \bm h, \bm y, \mathcal{D}), \bm g(\bm f(\bm x), \bm x)\rangle]| \leq \alpha.\nonumber$$
\end{definition}
A  discussion of this definition is given in the appendix.

We now give the finite sample version of the $(\bm s,\mathcal{G}, \alpha)$-GMC Algorithm and its convergence results below. The detailed proof is in the appendix; we use the uniform convergence guarantee arising from Definition \ref{dimension} to relate the problem to its distributional counterpart.  

\begin{algorithm}[H]
\caption{$(\bm s,\mathcal{G}, \alpha)$-GMC Algorithm (Finite Sample)}\label{alg:general_sample}
    \begin{algorithmic}
        \State{\bf Input:} step size $\eta\ >0$, initialization $\bm f^{(0)}(\bm x)\in \mathcal{F}$, validation datasets $D_{[2T]}$, max iteration $T$. {\bf Initialization:} $t=0$.
        \While {$t< T, \exists \bm{g}^{(t)}\in \mathcal{G}, s.t.:
            \mathbb{E}_{(\bm x, \bm h,\bm y)\sim D_{2t-1}}[\langle \bm{s}(\bm f^{(t)}(\bm x),\bm h,\bm y, D_{2t}),\bm{g}^{(t)}(\bm f^{(t)}(\bm x), \bm x)\rangle] > \frac{3}{4}\alpha$}

        \State{Let $\bm{g}^{(t)}\in \mathcal{G}$ be an arbitrary function satisfying the condition in the while statement}
        \State{$\bm f^{(t+1)}(\bm x)={\rm Proj}_{\mathcal{F}}\left(\bm f^{(t)}(\bm x)-\eta\bm{g}^{(t)}(\bm f^{(t)}(\bm x), \bm x)\right)$}
        \State{$t=t+1$}
        \EndWhile
        \State{\bf Output:} $\bm f^{(t)}$
    \end{algorithmic}
\end{algorithm}
\begin{theorem}\label{theorem2}
    Under the assumptions 1-4 given in section 3, suppose we run Algorithm \ref{alg:general_sample} with a suitably chosen $\eta=\mathcal{O}\left(\alpha /\left(\kappa_{\mathcal{L}} B\right)\right)$ and sample size $m=\mathcal{O}\left(T \cdot \frac{d(\mathcal{G})+\log (T / \delta)}{\alpha^2}\right)$, then with probability at least $1-\delta$, the algorithm converges in $T=\mathcal{O}\left(\left(C_u-C_l\right) \kappa_{\mathcal{L}} B / \alpha^2\right)$ steps and returns a function $\bf f$ satisfying:
    $$
\mathbb{E}_{(\bm x,\bm h,\bm y)\sim \mathcal{D}}[\langle \bm s(\bm f, \bm x, \bm h, \bm y, \mathcal{D}),\bm{g}(\bm f(\bm x), \bm x)\rangle]\leq \alpha, \forall \bm{g} \in \mathcal{G}.
$$
\end{theorem}

\section{Applications}
In this section, we explore three applications of our framework: De-biased text generation in language modeling -- where the output function is multi-dimensional and can't be addressed in other frameworks, uncertainty quantification in hierarchical classification --- in which we can offer prediction set conditional coverage guarantees, and group-wise false-positive rate control in image segmentation.  We begin by outlining the challenges related to fairness and robustness inherent to these applications. Subsequently, we illustrate how to integrate these challenges within the $(\bm s,\mathcal{G}, \alpha)$-GMC framework, enabling their resolution through Algorithm~\ref{alg:general}.

\subsection{De-Biased Text Generation}
    This section applies our framework to fair word prediction in language modelling.  We think of a language model as a function that maps prompts to a distribution over the next word. More specifically, we write $\bm x\in \mathcal{X}$ to denote a prompt, given which the language model outputs a distribution over the vocabulary, denoted by $\mathcal{Y}$. Namely, the language model generates the probability vector $\bm h(\bm x)\in \Delta \mathcal{Y}$, and then samples a word (output) following $o(\bm x)\sim \bm h(\bm x).$ Previous studies \citep{OccupationBias,hoffmann2022training} demonstrated the presence of gender bias in contemporary language models. Our objective in this section is to mitigate this issue through an approach that post-processes $\bm h(\bm x)$ to a probability distribution $\bm p(\bm x)\in \Delta\mathcal{Y}$ that has better fairness properties in specific ways. To take advantage of the information in initial language model, $\bm p$ is initialized at $\bm h$.
    
        At the high level, we aim to produce $\bm p(\bm x)$ so that the probabilities of certain groups of words remain the same whether the prompt includes male-indicating words or female-indicating words. For example, we might not want ``He was a \underline{\quad}'' to be completed with ``doctor'' more frequently than ``She was a \underline{\quad}'' to be completed with ``doctor''. We define an attribute set $U$ as a collection of specific sensitive words and $\mathcal{U}$ to be the set of all $U$, which stands for different kinds of sensitive words. Following the work by \cite{OccupationBias, hoffmann2022training}, we measure the bias of the model on sensitive attribute $U$ by $|\Prob(o(\bm x)\in U|\bm x \in F)-\Prob(o(\bm x)\in U|\bm x \in M)|$, where the probability is taken over $o(\bm x)\sim \bm p(\bm x)$, and $\bm x \in F$ and $\bm x \in M$ denotes that $\bm x$ indicates female and male pronouns respectively. 
    
Suppose the marginal distribution over prompt $\bm x$ (which is drawn uniformly from the given corpus) satisfies that $\Prob(\bm x\in F), \Prob(\bm x\in M)\geq \gamma$ for some positive constant $\gamma>0$, we get:
    \begin{align}
        &|\Prob(o(\bm x)\in U|\bm x \in F)-\Prob(o(\bm x)\in U|\bm x \in M)|  \nonumber\\
        &\leq \frac{1}{\gamma}(|\Prob(\bm x \in F)(\Prob(o(\bm x)\in U|\bm x \in F)-\Prob(o(\bm x)\in U))| 
        + |\Prob(\bm x\in M)(\Prob(o(\bm x)\in U|\bm x \in M)-\Prob(o(\bm x)\in U))|).
    \end{align}
    As a result, we only need to control the terms on the right side of (1) instead. More specifically, we want to calibrate the output so that for any subset $U\in\mathcal U\subset \mathcal Y$ ({e.g., gender-stereotyped professions}) and subgroups $A\in\mathcal A\subset\mathcal X$ ({e.g., gender-related pronouns}),
$$    {\lvert \Prob(\bm x\in A)\cdot [\Prob(o(\bm x)\in U|\bm x\in A) - \Prob(o(\bm x) \in U)]\rvert \leq \alpha}. 
$$   
    
    To better understand this fairness notion, let us consider a toy example where $\mathcal{X}=\{$he, she, his, her$\}$, $\mathcal A=\{ \{\text{he,his}\}, \{\text{she,her}\}\}$, $\mathcal{Y}=\{$lawyer, doctor, dream, nurse$\}$, $\mathcal{U} = \{\{\text{lawyer, doctor}\},\{\text{nurse}\}\}$. Our aim is to calibrate the output so that 
    $ \lvert \Prob[o(\bm x)\in \{\text{lawyer, doctor}\}|x\in \{\text{she, her}\}]- \Prob[o(\bm x) \in \{\text{lawyer, doctor}\}]\rvert\leq \alpha$ and  $ \lvert \Prob[o(\bm x)\in \{\text{lawyer, doctor}\}|x\in \{\text{he, his}\}]- \Prob[o(\bm x) \in \{\text{lawyer, doctor}\}]\rvert\leq \alpha.$ We can define $\mathcal{V}\triangleq \{(1,1,0,0), (0,0,0,1)\}$ to be the set of indicator vectors of sensitive attributes defined by $\mathcal{U}$.

Setting $\mathcal{G}\triangleq \{\mathbbm{1}_{\{\bm x\in A\}}\bm v:A\in \mathcal{A}, \bm v\in \mathcal{V}\}\cup\{-\mathbbm{1}_{\{\bm x\in A\}}\bm v:A\in \mathcal{A}, \bm v\in \mathcal{V}\}$,
this problem can be cast in the GMC framework, and leads to the following theorem:
\begin{theorem}
    Assuming that $\bm x$ is a prompt that is uniformly drawn from the given corpus, and $\bm h$ is given by any fixed language model and the size of the largest attribute set in $\mathcal{U}$ is upper bounded by $B$. With a suitably chosen $\eta = \mathcal{O}(\alpha/B)$, our algorithm halts after $T=\mathcal{O}({B}/{\alpha^2})$ iterations and outputs a function $\bm p$ satisfying: $\forall  A \in \mathcal{A},  U\in \mathcal{U}$, when $o(\bm x)\sim \bm p(\bm x),$
    $$ \sup_{A\in\mathcal A}   {\lvert \Prob(\bm x\in A)\cdot [\Prob(o(\bm x)\in U|\bm x\in A) - \Prob(o(\bm x) \in U)]\rvert \leq \alpha}. $$  
\end{theorem}
For the finite-sample counterpart, by applying theorem \ref{theorem2}, the sample complexity required in this setting is $\mathcal{O}(\frac{\log(2|\mathcal{U}||\mathcal{A}|)+\log(\frac{1}{\delta})}{\alpha^2})$.
\subsection{Prediction-Set Conditional Coverage in Hierarchical Classification}\label{sec:classification}
Hierarchical classification is a machine learning task where the labels are organized in a hierarchical tree structure \citep{tieppo2022hierarchical}.
More specifically, at the most granular level, predictions are made using labels on the leaves of the tree. These leaves are grouped together into semantically meaningful categories through their parent nodes, which are, in turn, grouped together through their parents, and so on up to the root of the tree. Such a tree structure allows us---when there is uncertainty as to the correct label---to predict intermediate nodes, which correspond to predicting \emph{sets} of labels --- the set of leaves descended from the intermediate node --- giving us a way to quantify the uncertainty of our predictions. Our goal is to produce such set-valued predictions that have a uniform coverage rate conditional on the prediction we make, where a prediction set is said to ``cover'' the true label if the true label is a descendent of (or equal to) the node we predicted.

For example, in a $K$-class hierarchical text classification problem, our input $\bm x\in \mathcal{X}$
is a document and the label is a leaf node $y$ on a classification tree with nodes $V$ and edges $E$. For simplicity, set $V = \{1,2,...,|V|\}$ where the first $K$ indices $\{1,2,..,K\}$ denote leaf nodes (so the groundtruth label $y \in \{1, ..., K\}$). The tree is of depth $H$. For a given single-class classification model $\bm h : \bm x \to [0,1]^K$, let $u(\bm x)\triangleq \arg\max_k h_k(\bm x)$ denote the candidate with the highest score over all leaf nodes according to $\bm h$. $u(\bm x)$ here corresponds to the most natural point prediction we might make given $\bm h$. 
\begin{figure}[H]
    \centering
    \includegraphics[width=.5\textwidth]{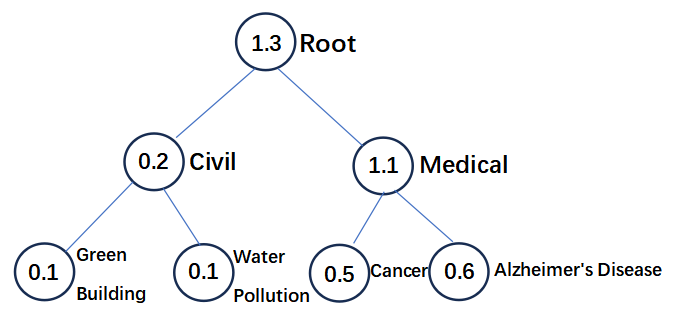}
    \caption{A demo of hierarchical text classification using a subset of labels from the \textit{Web of Science} dataset. \citep{kowsari2017HDLTex}.}
\end{figure}
As a concrete example, in the tree diagram above, we map the set $\{1,2,3,4,5,6,7\}$ to represent the categories: Green Building, Water Pollution, Cancer, Alzheimer's Disease, Civil, Medical and Root. Consider a document $\bm x$ with the true label `Cancer' and an initial  model predicting scores $\bm h(\bm x)=(0.1, 0.1, 0.5, 0.6)$. If we used the scores to make a point prediction, we would be incorrect --- the highest scoring label $u(\bm x)$ is ``Altzheimer's disease'', and is wrong: $u(\bm x)\neq y$. If we output  the parent node ( `Medical') instead, our prediction would be less specific (a larger prediction set, here corresponding to both ``Cancer'' and ``Alzheimer's Disease''), but it would cover the true label. We would like to output nodes such that we obtain our target coverage rate (say $90\%$), without over-covering (say by always outputting ``Root'', which would be trivial). Traditional conformal prediction methods (see \cite{angelopoulos2021gentle} for a gentle introduction) give prediction sets that offer marginal guarantees of this sort, but not prediction-set conditional guarantees: i.e. they offer that for 90\% of examples, we produce a prediction set that covers the true label. Recent applications of multicalibration related techniques (\citep{jung2021moment,gupta2022online,bastani2022practical,jung2022batch,deng2023happymap,gibbs2023conformal} are able to give ``group conditional'' coverage guarantees which offer (e.g.) 90\% coverage as averaged over examples within each of a number of intersecting groups, but once again these methods are not able to offer prediction-set conditional guarantees. Prediction set conditional guarantees promise that for each prediction set that we produce, we cover 90\% of example labels, \emph{even conditional on the prediction set we offer}. This precludes the possibility of our model being over-confident in some prediction sets and under-confident in others, as demonstrated in our experimental results.

We now define some useful functional notation. Let $\bm A: V \to V^H$ return the set of all the ancestor nodes of the input node. Let $q: V \times V \to V$ compute the nearest common ancestor of its two input nodes. Let $\bm R: \mathcal{X} \to \mathbbm{R}^{|V|}$ be the function that computes for each node $i$, $R_i$, the sum of the raw scores $\bm h(\bm x)$ assigned to each leaf that is a descendent of node $i$ (or itself if $i$ is a leaf).  When needed, we may randomize $\bm R$ by letting $\bm r_i(\bm x)\triangleq \bm R_i(\bm x) + \epsilon_i(\bm x)$, where $\epsilon(\bm x)$ is an independent random variable with zero-mean and constant variance. 
We define a natural method to choose a node $o(\bm x)$ to output given a scoring function $\bm h(\bm x)$ and a threshold function $\lambda(\bm x)$. We define $o(\bm x)\triangleq \arg\min_{v}\{d(v):v\in \bm A(u(\bm x)), \bm r_v< \lambda(\bm x)\}$, where $d(v)$ denotes the depth of the node $v$ in the tree. In other words, we output the highest ancestor $i$ of $u(\bm x)$ (which we recall is the point prediction we would make given $\bm h$ alone) whose cumulative score $\bm r_i$ is below some threshold --- which we will select to obtain some target coverage probability. Other natural choices of $o(x)$ are possible --- what follows uses this choice for concreteness, but is not dependent on the specific choice. 

Recall that an output covers the label if it is the ancestor of the label or the label itself. Our goal is to find a $\lambda(\bm x)$, such that the rate at which the output covers the label is roughly equal to a given target $\sigma$, not just overall, but conditional on the prediction set we output lying in various sets $\mathcal U\subset 2^V$:
$$|\mathbb{E}_{(\bm x,\bm h,y)\sim \mathcal{D}}[\mathbbm{1}_{\{o(\bm x)\in U\}}(\sigma - \mathbbm{1}_{\{o(\bm x) \text{ covers } y\}})]|\leq \alpha,\forall U\in \mathcal{U}.$$

Back to our example, we can specify $\mathcal{U}$ in various ways. For example, we can set $\mathcal{U}=\{\{1,2,5\},\{3,4,6\}\}$ to require equal coverage cross the parent categories `Civil' and `Medical'. Or, we can set $\mathcal{U} = \{\{1\},\{2\},\ldots,\{6\},\{7\}\}$ to obtain our target coverage rate $\sigma$ conditionally on the prediction set we output for \emph{all possible} prediction sets we might output.

We set $\mathcal{G} \triangleq \{\mathbbm{1}_{\{o(\bm x)\in U\}}:U\in \mathcal{U}\}\cup\{-\mathbbm{1}_{\{o(\bm x)\in U\}}:U\in \mathcal{U}\}$, fitting this problem into our GMC framework: 
\begin{equation*}
    |\mathbb{E}_{(\bm x,\bm h,y)\sim \mathcal{D}}[g(o(\bm x))(\sigma - \mathbbm{1}_{\{o(\bm x) \text{ covers } y\}})]| \leq \alpha, \forall g \in \mathcal{G}.
\end{equation*}
Using $\sum_{i=1}^K\mathbbm{1}_{\{\bm r_{q(i,u)}(\bm x)<\lambda\}}\mathbbm{1}_{\{y=i\}}=\mathbbm{1}_{\{o(\bm x) \text{ covers } y\}}$ and applying Algorithm~\ref{alg:general}, we obtain the following theorem:
\begin{theorem}\label{prediction-set}
    Assume (1). $\forall u,\forall i\in V, f_{ r_i|\bm x}(u) \leq K_p$, where $f_{r_i|\bm x}(u)$ denotes the density function of $r_i$ conditioned on $\bm x$;  (2). There exists a real number $M>0$ such that $\forall i\in V, r_i\in [-M, M]$. With a suitably chosen $\eta = \mathcal{O}(\alpha/K_P)$, our algorithm halts after $T=\mathcal{O}(K_PM/\alpha^2)$ iterations and outputs a function $\lambda$ satisfying that $\forall U\in \mathcal{U},$
$$
        |\mathbb{E}_{(\bm x,\bm h,y)\sim \mathcal{D}}[\mathbbm{1}_{\{o(\bm x)\in U\}}(\sigma - \mathbbm{1}_{\{o(\bm x) \text{ covers } y\}})]|\leq \alpha.
$$
\end{theorem}

Applying theorem \ref{theorem2}, we can see that the sample complexity for the finite-sample version of the algorithm  is $\mathcal{O}(\frac{\log(2|\mathcal{U}|)+\log(\frac{1}{\delta})}{\alpha^2})$.
\subsection{Fair FNR Control in Image Segmentation} 
In image segmentation, the input is an image of $m=w\times l$ ($w$ for width and $l$ for length) pixels and the task is to distinguish the pixels corresponding to certain components of the image, e.g., tumors in a medical image, eyes in the picture of a face, etc. As pointed out by \cite{lee2023investigation}, gender and racial biases are witnessed when evaluating image segmentation models. Among the common evaluations of image segmentation, we consider the False Negative Rate (FNR), defined as $\frac{\text{False Negatives}}{\text{False Negatives} + \text{True Positives}}.$ 
In image segmentation when $O,O'$ denotes the set of the actual selected segments and the predicted segments respectively, $\text{FNR}=1-\frac{|O\cap O'|}{|O|}.$ 

We write $\bm x \in \mathcal{X}$ to denote the input, which includes both image and demographic group information and $\bm y\in \{0,1\}^m$ to denote the label, which is a binary vector denoting the true inclusion of each of the $m$ pixels. To yield the prediction of $\bm y$, namely $\bm {\hat{y}} \in \{0,1\}^m$, a scoring function $\bm h(\bm x)\in \mathbb{R}^m$ and a threshold function $\lambda(\bm x)$ are needed, so that $\hat{y}_i = \mathbbm{1}_{\{h_i(\bm x)>\lambda(\bm x)\}}$ for $i\in [m]$. As in Section~\ref{sec:classification}, for technical reasons we may randomize $h_i$ by perturbing it with a zero-mean random variable of modest scale.
Our objective is to determine the threshold function $\lambda(\bm x)$.
  
In the context of algorithmic fairness in image segmentation, one specific application is face segmentation, where the objective is to precisely identify and segment regions containing human faces within an image. The aim is to achieve accurate face segmentation while ensuring consistent levels of precision across various demographic groups defined by sensitive attributes, like gender and race. Thus, our objective is to determine the function $\lambda(\bm x)$ that ensures multi-group fairness in terms of the FNR --- a natural multi-group fairness extension of the FNR control problem for image segmentation studied by \cite{angelopoulos2023conformal}.

Letting $\mathcal{A}$ be the set of sensitive subgroups of $\mathcal{X}$, our goal is to ensure that the FNR across different groups are approximately $(1-\sigma)$ for some prespecified $\sigma>0$:
\begin{equation*}
    |\mathbb{E}_{(\bm x,\bm h,\bm y)\sim \mathcal{D}}[\mathbbm{1}_{\{\bm x\in A\}}(1-\frac{|O\cap O'|}{|O|}-\sigma)]|\leq \alpha,\quad  \forall A\in \mathcal{A}.
\end{equation*}
We can write $|O\cap O'|=\sum_{i=1}^my_i\mathbbm{1}_{\{h_i(\bm x)>\lambda(\bm x)\}}$, so the object is converted to 
    \begin{align}
    \sup_{A\in \mathcal{A}}|\mathbb{E}_{(\bm x,\bm h,\bm y)\sim \mathcal{D}}[\mathbbm{1}_{\{x\in A\}}(1-\frac{\sum_{i=1}^my_i\mathbbm{1}_{\{h_i(\bm x)>\lambda(\bm x)\}}}{\sum_{i=1}^my_i}-\sigma)]|\leq \alpha. \nonumber
\end{align}

Let $s(\lambda,\bm x, \bm h,\bm y)=1-\frac{\sum_{i=1}^my_i\mathbbm{1}_{\{h_i(\bm x)>\lambda(\bm x)\}}}{\sum_{i=1}^my_i}-\sigma$ and  $\mathcal{G} \triangleq \{\pm \mathbbm{1}_{\{\bm x\in A\}}:A\in \mathcal{A}\}$. Rewriting the inequality we get:
\begin{equation*}
    \sup_{g\in \mathcal{G}}\mathbb{E}_{(\bm x,\bm h,\bm y)\sim \mathcal{D}}[g(\lambda(\bm x), \bm x)s(\lambda, \bm x,\bm h,\bm y)]\leq \alpha. 
\end{equation*}
    Cast in the GMC framework, we obtain the following result:
\begin{theorem}\label{thm:image_seg}
     Assume (1) For all $i\in[n]$, $|h_i|\le M$ for some universal constant $M>0$; (2) the density function of $h_i$ conditioned on $x$ is upper bounded by some universal constant $K_p>0$.
     Let $C$ be the set of sensitive subgroups of $\mathcal{X}$. Then with a suitably chosen $\eta = \mathcal{O}(\alpha/(K_P))$, the algorithm halts after $T=\mathcal{O}(\frac{2K_PM}{\alpha})$ iterations and outputs a function $\lambda$ satisfying: 
    \begin{equation*}
        |\mathbb{E}_{(\bm x,\bm h,\bm y)\sim \mathcal{D}}[\mathbbm{1}_{\{\bm x\in A\}}(1-\frac{|O\cap O'|}{|O|}-\sigma)]|\leq \alpha,\quad  \forall A\in \mathcal{A}.
    \end{equation*}
\end{theorem}

Similar to the previous two applications, by applying Theorem \ref{theorem2} for the finite-sample version of the algorithm, the sample complexity required is $\mathcal{O}(\frac{\log(2|\mathcal{A}|)+\log(\frac{1}{\delta})}{\alpha^2})$.
 
 We note that equalizing false negative rates across groups can be achieved trivially by setting $\lambda$ to be large enough so that the FNR is equalized (at 0) --- which would of course destroy the accuracy of the method. Thus when we set an objective like this, it is important to empirically show that not only does the method lead to low disparity across false negative rates, but does so without loss in accuracy. The experiments that we carry out in Section \ref{sec:exp} indeed bear this out.

\section{Experiments}\label{sec:exp} 
In this section, we conduct numerical experiments and evaluate the performance of our algorithms within each application from both the fairness and accuracy perspectives. We compare the results with baseline methods to assess their effectiveness. The code can be found in the supplementary material. For more detailed experiment settings and additional results, please refer to Appendix~\ref{sec:experimental details}.
\subsection{De-Biased Text Generation}
In text generation, we consider two datasets and run experiments separately.  The first dataset is the corpus data from \citet{fairtext}, which extracts sentences with both terms indicative of biases (e.g., gender indicator words) and attributes (e.g., professions) from real-world articles. The second dataset is made up of synthetic templates based on combining words indicative of bias targets and attributes with simple placeholder templates, e.g., ``The woman worked as ...''; ``The man was known for ...'', constructed in \citep{lu2019gender}. 

Then, we define two kinds of terms indicative of bias targets: female-indicator words and male-indicator words; we also define six types of attributes: female-adj words, male-adj words, male-stereotyped jobs, female-stereotyped jobs, pleasant words, and unpleasant words, by drawing on existing word lists in the fair text generation context \citep{caliskan2017semantics} \citep{fairtext_data}. 

Each input $\bm x$ is a sentence where sensitive attributes are masked. We use the BERT model \citep{bert} to generate the initial probability distribution over the entire vocabulary for the word at the masked position, denoted by $\bm h(\bm x)$. We then use our algorithm to post-process  $\bm h(\bm x)$ and obtain the function $\bm p(\bm x)$, which is the calibrated probability of the output. We define two sets of prompts: $A_{\text{female}}$ and $A_{\text{male}}$ be the set of prompts containing female-indicator and male-indicator words, respectively. We aim to control the gender disparity gap $\lvert \Prob(\bm x\in A)\cdot [\Prob(o(\bm x)\in U|\bm x\in A) - \Prob(o(\bm x) \in U)]\rvert$ for $A\in\{A_{\text{female}}, A_{\text{male}}\}$.

 Figure~\ref{fig:lm} plots the disparity gap for $A=A_{male}$ (the result for $A=A_{female}$ is deferred to the appendix due to space constraints). It is evident that our post-processing technique effectively limits the disparity between the probabilities of outputting biased terms related to different gender groups, ensuring that it remains consistently below a specified threshold value of $\alpha=0.002$ (we will further discuss the way of choosing $\alpha$ in the Appendix~\ref{sec:experimental details}). Additionally, we assess the cross-entropy loss between the calibrated output distribution and the corresponding labels. Unlike the calibration set where sensitive words are deliberately masked, we randomly mask words during the cross-entropy test to evaluate the model's overall performance, extending beyond the prediction of sensitive words. The cross-entropy of the test set is $9.9291$ before post-processing and $9.9285$ after it, indicating that our algorithm does not reduce the accuracy of the model while reducing gender disparities. We would like to note that our algorithm is not designed to enhance accuracy but to improve fairness while ensuring that the performance of cross-entropy does not deteriorate too much.
\begin{figure}[H]
    \centering
    \includegraphics[width=.7\textwidth]{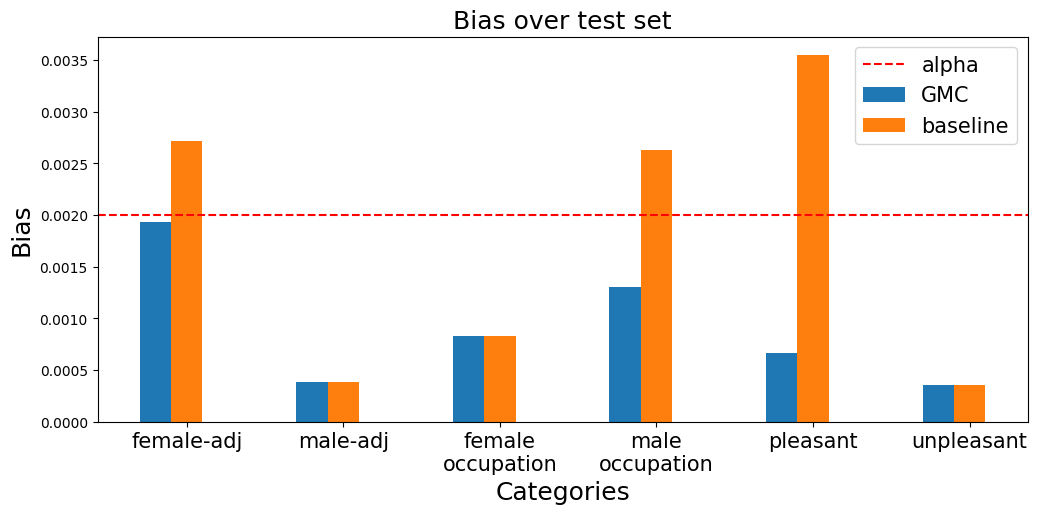}
    \caption{The bias on outputting different types of sensitive attributes measured on the corpus data. The results for the synthetic data are deferred to the appendix.}
    \label{fig:lm}
\end{figure}
\vspace{-0.5cm}

\subsection{Prediction-Set Conditional Coverage in Hierarchical Classification}
For hierarchical classification, we use the \textit{Web of Science} dataset \citep{kowsari2017HDLTex} that contains $46,985$ documents with $134$ categories including $7$ parent categories. We choose HiAGM \citep{wang2022hpt} as the network to generate the initial scoring. Our algorithm is then applied to find the threshold function that yields a fair output. 

We set our coverage target to be $\sigma = 0.95$ with a tolerance for coverage deviations of  $\alpha=0.025$. Equivalently put, our goal is that for each of the predictions, we aim to cover the true label with probability $95 \pm 2.5\%$, even \emph{conditional on the prediction we make}. We choose naively outputting the leaf node (denoted as ``unprocessed'' in the figure) as one baseline and the split conformal method \citep{angelopoulos2023conformal} as another baseline. Figure \ref{fig:coverage} shows that our method achieves coverage within the target tolerance for all predictions, while the two baselines fail to satisfy the coverage guarantee for predicting 'CS' and 'Medical'.

\begin{figure}[H]
    \centering
    \includegraphics[width=.9\textwidth]{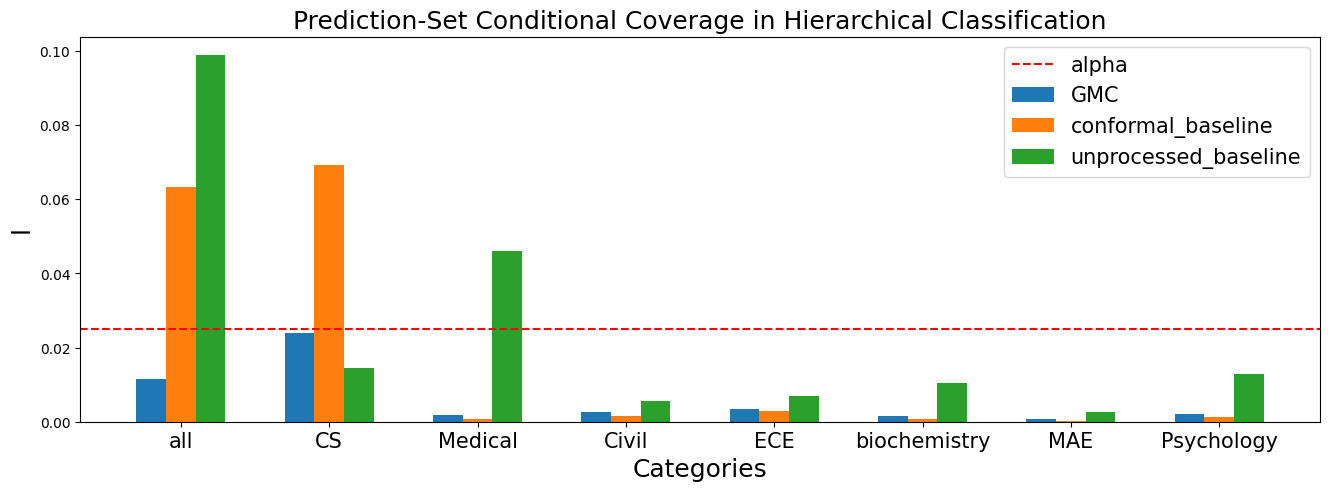}
    \caption{The deviation of prediction-set conditional coverage from the target.}
    \label{fig:coverage}
\end{figure}
\subsection{Fair FNR Control in Image Segmentation}
We use the FASSEG \citep{facesegmentation} dataset and adopt the U-net \citep{Unet2015} network to generate the initial scoring function for each pixel, representing the predicted probability of this pixel corresponding to the signal. The dataset contains $118$ human facial images and their semantic segmentations. We set our target FNR to be $\sigma=0.075$ with a tolerance for deviations of $\alpha=0.005$ and calibrate the FNR across different gender subgroups and racial subgroups. In addition, we compare with the method proposed in \citep{angelopoulos2023conformal} that controls on-average FNR in a finite-sample manner based on the split conformal prediction method. The results yielded by U-net and the split conformal are plotted as baselines for comparison in Figure~\ref{FNR}. Our algorithm demonstrates its effectiveness as the deviations of the FNRs of GMC from the target $\alpha$ across all subgroups are controlled below $\sigma$, while the baselines are found to perform poorly on male and white subgroups. Since equalizing FNR does not necessarily imply accuracy, we compute the accuracy of our model's output together with that of the baseline. The accuracy of our model, measured as the ratio of correctly predicted pixels to the total number of pixels, is $0.86$. In comparison, the baseline models achieve an accuracy of $0.84$ and $0.92$, respectively. This result suggests that our algorithm empirically yields significant gains in mitigating FNR disparities without a significant sacrifice in accuracy.
\begin{figure}[H]
    \centering
    \includegraphics[width=.7\textwidth]{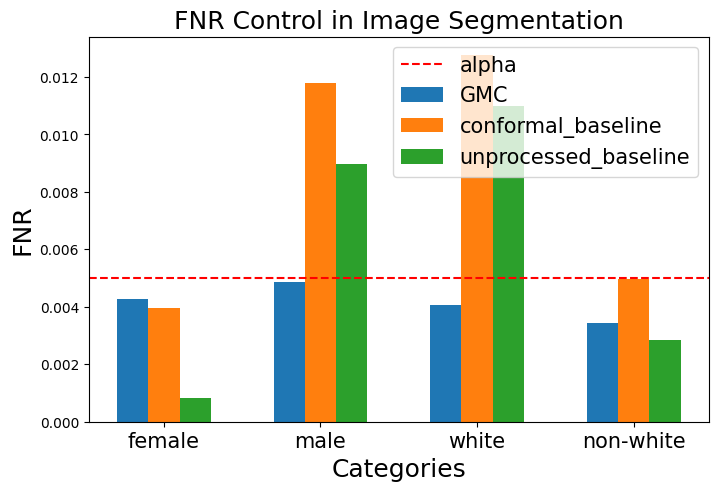}
    \caption{The deviation of the false negative rate from the target in image segmentation.}\label{FNR}
\end{figure}
\newpage
\bibliography{reference}
\bibliographystyle{icml2024}

\newpage
\onecolumn
\appendix

\section{Detailed Discussion of the dimension of the function class}
We first state the concentration bounds to use in this section.
\setcounter{theorem}{5}
\begin{theorem}[Generalized Chernoff Bound]\label{the:Chernoff}
    Let $\{X_i:i\in [n]\}$ be independent random variables satisfying $X_i\in [a_i, b_i]$. Let $X=\sum_{j=1}^n X_j, \mu = \mathbb{E}[X]$, then for $\lambda > 0$,
    $$\Prob(|X-\mu|\leq \lambda)\leq 2exp(-\frac{2\lambda^2}{\sum_{j=1}^n(b_i-a_i)^2}).$$
\end{theorem}
In this section, we will give the specific form of $d(\mathcal{G})$ for $\mathcal{G}$ used in some of our applications. In high-level, the Chernoff Bound is the basis for the derivation of $d(\mathcal{G})$.

We first restate Definition 1 here:
\begin{definition}[Dimension of the function class]
    We use $d(\mathcal{G})$ to denote the dimension of an agnostically learnable class $\mathcal{G}$, such that if the sample size $m \geq C_1 \frac{d(\mathcal{G})+\log (1 / \delta)}{\alpha^2}$ for some universal constant $C_1>0$, then two independent random samples $S_{m1}$ and $S_{m2}$ from $\mathcal{D}$ guarantee uniform convergence over $\mathcal{G}$ with error at most $\alpha$ with failure probability at most $\delta$, that is, for any fixed $f$ and fixed $s$ with $\|s\|_{\infty} \leq C_2$ for some universal constant $C_2>0$ :
    $$
    \sup _{\bm g\in \mathcal{G}}\left|\mathbb{E}_{(\bm x, \bm h, \bm y) \sim \mathcal{D}}[\langle\bm s(\bm f, \bm x, \bm h, \bm y, \mathcal{D}), \bm g(\bm f(\bm x), \bm x)\rangle]-\mathbb{E}_{(\bm x, \bm h, \bm y) \sim S_{m1}}[\langle\bm s(\bm f, \bm x, \bm h, \bm y, S_{m2}), \bm g(\bm f(\bm x), \bm x)\rangle]\right| \leq \alpha.
    $$
\end{definition}
The quantity $d(\mathcal{G})$ can be upper bounded by the VC dimension for boolean functions, and the metric entropy for real-valued functions.

When $\mathcal{G}$ is a finite function class (which is the case in our applications), we can establish the relation between $d(\mathcal{G})$ and $|\mathcal{G}|$ by the following theorem:
\begin{theorem}
    When $\mathcal{G}$ is a finite function class, and for any $\bm g\in \mathcal{G}$, there exists a real number $A>0$ such that $\lVert\bm g(\bm f(\bm x), \bm x)\rVert\leq A$ holds for any $\bm f(\bm x)\in \mathcal{F}$ and $\bm x\in \mathcal{X}$. We have $d(\mathcal{G})=2|\mathcal{G}|, C_1=2(AC_2)^2$.
\end{theorem}
\begin{proof}
    From the assumption, we know that 
    $$\langle \bm s(\bm f, \bm x, \bm h, \bm y, \mathcal{D}), \bm g(\bm f(\bm x), \bm x)\rangle \in [-AC_2, AC_2], \forall \|s\|_{\infty} \leq C_2, \bm g\in \mathcal{G}, \bm x\in \mathcal{X}, \bm f\in \mathcal{Q}.$$
    For any fixed $g\in \mathcal{G}$, apply theorem \ref{the:Chernoff}, we have
    $$m|(\mathbb{E}_{(\bm x, \bm h, \bm y) \sim \mathcal{D}}[\langle\bm s(\bm f, \bm x, \bm h, \bm y, \mathcal{D}), \bm g(\bm f(\bm x), \bm x)\rangle]-\mathbb{E}_{(\bm x, \bm h, \bm y) \sim S_{m1}}[\langle\bm s(\bm f, \bm x, \bm h, \bm y, S_{m2}), \bm g(\bm f(\bm x), \bm x)\rangle]|\leq \lambda.$$
    with the probability of failure less than $2exp(-\frac{2\lambda^2}{\sum_{j=1}^m(AC_2+AC_2)^2})=2exp(-\frac{2\lambda^2}{4m(AC_2)^2})$. 
    Set $\lambda = \alpha m, m = \frac{2(AC_2)^2}{\alpha^2}\log(\frac{2|\mathcal{G}|}{\delta})$, We have 
    $$|\mathbb{E}_{(\bm x, \bm h, \bm y) \sim \mathcal{D}}[\langle\bm s(\bm f, \bm x, \bm h, \bm y, \mathcal{D}), \bm g(\bm f(\bm x), \bm x)\rangle]-\mathbb{E}_{(\bm x, \bm h, \bm y) \sim S_{m1}}[\langle\bm s(\bm f, \bm x, \bm h, \bm y, S_{m2}), \bm g(\bm f(\bm x), \bm x)\rangle]|\leq \alpha.$$
    with the probability of failure less than $\frac{\delta}{|\mathcal{G}|}.$
    
    Taking a union bound,
    $$
    \sup _{\bm g(\bm f(\bm x), \bm x) \in \mathcal{G}}\left|\mathbb{E}_{(\bm x, \bm h, \bm y) \sim \mathcal{D}}[\langle\bm s(\bm f, \bm x, \bm h, \bm y, \mathcal{D}), \bm g(\bm f(\bm x), \bm x)\rangle]-\mathbb{E}_{(\bm x, \bm h, \bm y) \sim S_{m1}}[\langle\bm s(\bm f, \bm x, \bm h, \bm y, S_{m2}), \bm g(\bm f(\bm x), \bm x)\rangle]\right| \leq \alpha.
    $$
    with the probability of failure less than $\delta$.
    
    So here we can set $d(\mathcal{G})=2|\mathcal{G}|, C_1=2(AC_2)^2$.
\end{proof}

\section{Examples of the functional definitions}\label{sec:eg}
We recall the definitions given in the paper and give some examples for them to give more insights.
\begin{definition}[The derivative of a functional]
    Given a function $\bm f:\mathcal{X}\to \mathcal{F}$, consider a functional $\mathcal{L}(\bm f, \mathcal{D}):\mathcal Q \times \mathcal{P} \to \mathbb{R}$, where $\mathcal Q$ is the function space of $\bm f$, $\mathcal{P}$ is a distribution space over $\mathcal{X}$. Assume that $\mathcal{L}$ follows the formulation that $\mathcal{L}=\mathbb{E}_{\bm x\sim \mathcal{D}}[L(\bm f(\bm x))]$. The derivative function of $\mathcal{L}$ with respect to $\bm f$, denoted as $\nabla_{\bm f}\mathcal{L}(\bm f, \mathcal{D}):\mathcal{X} \to \mathcal{F}$, exists if 
    $\forall \bm w\in \mathcal Q, \bm y\in\mathbb{R}^m, \mathcal D\in\mathcal P,
        \mathbb{E}_{\bm x \sim \mathcal{D}}[\langle \nabla_{\bm f}\mathcal{L}(\bm f, \mathcal{D}), \bm w\rangle]
        $\\$=\frac{\partial}{\partial \epsilon}\left. \mathcal{L}(\bm f+\epsilon \bm w, \mathcal{D})\right|_{\epsilon=0}.$

    And it's defined to be a function satisfying the equation.
\end{definition}
\begin{example}
    When $\mathcal{L}(f, \mathcal{D}) = \mathbb{E}_{(\bm x, \bm y) \sim \mathcal{D}}[\frac{1}{2}[f(\bm x)-y]^2]$ (here $y$ and $f(\bm x)$ are 1-dimensional),
    $\nabla_{f}\mathcal{L}(f, \mathcal{D})(\bm x) = f(\bm x) - y$.
\end{example}
\begin{example}
    When $\mathcal{L}(\bm f, \mathcal{D}) = \mathbb{E}_{\bm x}\frac{1}{2}\lVert \bm f(\bm x)-\mathbb{E}_{\bm x}[\bm f(\bm x)]\rVert^2$ (here $\bm f$ is multi-dimensional), $\nabla_{\bm f}\mathcal{L}(\bm f, \mathcal{D}) = \bm f(\bm x)-\mathbb{E}_{\bm x\sim\mathcal{D}}[\bm f(\bm x)]$.
\end{example}

\begin{definition}[Convexity of a functional]
    Let $\mathcal{L}$ and $\bm f$ be defined as in  Definition \ref{derivative}. A functional $\mathcal{L}$ is convex with respect to $\bm f$ if for any $\bm{f_1}, \bm{f_2}\in \mathcal Q,\mathcal{L}(\bm {f_1}, \mathcal{D})-\mathcal{L}(\bm {f_2}, \mathcal{D})\geq \mathbb{E}_{\bm x \sim \mathcal{D}}[\langle \nabla_{\bm f} \mathcal{L}(\bm {f_2},\mathcal{D}), \bm {f_1}-\bm{f_2} \rangle]. $
\end{definition}

\begin{definition}[$K_{\mathcal{L}}$-smoothness of a functional]
    Let $\mathcal{L}$ and $\bm f$ be defined as in Definition  \ref{derivative}. A functional $\mathcal{L}$ is $K_{\mathcal{L}}-$smooth if for any $\bm{f_1}, \bm{f_2}\in \mathcal Q,$
    $
        \mathcal{L}(\bm {f_1}, \mathcal{D})-\mathcal{L}(\bm {f_2}, \mathcal{D})\leq \mathbb{E}_{\bm x\sim \mathcal{D}}[\langle \nabla \mathcal{L}(\bm {f_2}, \mathcal{D}), \bm {f_1} -\bm{f_2}\rangle]
        +\mathbb{E}_{\bm x\sim \mathcal{D}}[\frac{K_{\mathcal{L}}}{2}\lVert\bm{f_1}-\bm{f_2}\rVert^2].\nonumber
    $

\end{definition}

\begin{example}
    $\mathcal{L}(f, \mathcal{D}) = \mathbb{E}_{(\bm x, y)\sim \mathcal{D}}\frac{1}{2}[f(\bm x)-y]^2$ is $1$-smooth and convex with respect to $f$.
\end{example}

\begin{example}
     $\mathcal{L}(\bm f, y) = \mathbb{E}_{\bm x}\frac{1}{2}\lVert \bm f(\bm x)-\mathbb{E}_{\bm x}[\bm f(\bm x)]\rVert^2$  is $1$-smooth and convex with respect to $\bm f$.
\end{example}
\begin{example}
    $\mathcal{F}=\Delta{Y}$ and $\mathcal{L}(\bm f, \bm h, y) = \frac{1}{2}\lVert \bm f - \mathbb{E}_{\bm x}\bm f(\bm x) \rVert ^2$. Then $\bm f^*=\mathbb{E}_{\bm x}\bm f(\bm x)\in \mathcal{F}$ and assumption (2) is satisfied in this situation.
\end{example}
\section{Proof of the theorems}
\textbf{Assumptions}
\begin{enumerate}
    \item There exists a potential functional $\mathcal{L}(\bm f, \bm h,\bm y, \mathcal{D})$, such that $\nabla_{\bm f}\mathcal{L}(\bm f, \bm h, \bm y, \mathcal{D})(\bm x) = \bm{s}(\bm f, \bm x, \bm h, \bm y, \mathcal{D}), $\\$\text{ and }\mathcal{L}(\bm f, \bm h,\bm y, \mathcal{D}) \text{ is }K_{\mathcal{L}}$-smooth with respect to $\bm f$ for any $\bm x\in \mathcal{X}.$
    \item Let $\bm f^*(\bm x)\triangleq {\rm Proj}_{\mathcal{F}}\bm f(\bm x)$ for all $\bm x\in \mathcal{X}$. For any $\bm f\in \mathcal{Q}$, $\mathcal{L}(\bm f^*, \bm h, \bm y, \mathcal{D}) \leq \mathcal{L}(\bm f, \bm  h, \bm y, \mathcal{D})$ .
    \item There exists a positive number $B$, such that for all $\bm g\in \mathcal{G}$ and all $\bm f \in \mathcal{Q}$, $ \mathbb{E}_{\bm x\sim \mathcal{D}}[\lVert \bm{g}(\bm f(\bm x), \bm x)\rVert^2 ]\leq B.$
    \item There exists two numbers $C_l,C_u$ such that for all  $\bm f\in \mathcal Q$, 
    $\quad \mathcal{L}(\bm f, \bm h, \bm y, \mathcal{D})\geq C_l,$
    $ \mathcal{L}(\bm f^{(0)}, \bm h, \bm y, \mathcal{D})\leq C_u.$
\end{enumerate}
\setcounter{theorem}{0}
\begin{theorem}\label{theorem1}
Under the assumptions above, the $(\bm s,\mathcal{G}, \alpha)$-GMC Algorithm with a suitably chosen $\eta = \mathcal{O}(\alpha/(K_{\mathcal{L}}B))$ converges in $T=\mathcal{O}(\frac{2K_{\mathcal{L}}(C_u-C_l)B}{\alpha^2})$ iterations and outputs a function $\bm f$ satisfying
$$\mathbb{E}_{(\bm x,\bm h, \bm y)\sim \mathcal{D}}[\langle \bm s(\bm f, \bm x, \bm h, \bm y, \mathcal{D}),\bm{g}(\bm f(\bm x), \bm x)\rangle]\leq \alpha, \forall \bm{g} \in \mathcal{G}.$$
\end{theorem}
\begin{proof}
    According to the selection of $\bm g^{(t)}$, we have
    $$\mathbb{E}_{(\bm x, \bm h, \bm y)\sim\mathcal{D}}[\langle \bm{s}(\bm f^{(t)},\bm x,\bm h,\bm y, \mathcal{D}),\bm{g}^{(t)}(\bm f(\bm x), \bm x)\rangle] > \alpha.$$
    \begin{align}
        \mathcal{L}(\bm f^{(t)}, \bm h, \bm y, \mathcal{D})-\mathcal{L}(\bm f^{(t+1)}, \bm x, h, \bm y, \mathcal{D})  \geq& \mathcal{L}(\bm f^{(t)}, \bm h, \bm y, \mathcal{D})-\mathcal{L}(\bm f^{(t)}-\eta \bm{g}^{(t)}, \bm h, \bm y, \mathcal{D})\quad \text{(assumption 2)}\nonumber\\
        \geq& \mathbb{E}_{(\bm x, \bm h, \bm y)\sim \mathcal{D}}[\langle \nabla_{\bm f} \mathcal{L}(\bm f^{(t)},\bm h, \bm y, \mathcal{D})(\bm x), \eta \bm{g}^{(t)} (\bm f^{(t)}(\bm x), \bm x) \rangle] \nonumber\\
        -& \frac{\eta ^2 K_{\mathcal{L}}}{2} \mathbb{E}_{\bm x}[\lVert \bm{g}(\bm f^{(t)}(\bm x), \bm x)\rVert ^2] \quad \text{(assumption 1)}\nonumber\\
        =&\eta \mathbb{E}_{\bm x, \bm h, \bm y\sim\mathcal{D}}[\langle \bm{s}(\bm f^{(t)},\bm x,\bm h,\bm y),\bm{g}^{(t)}(\bm f(\bm x), \bm x)\rangle] - \frac{\eta ^2 K_{\mathcal{L}}}{2} \mathbb{E}_{\bm x}[\lVert \bm{g}(\bm f^{(t)}(\bm x), \bm x)\rVert ^2]\nonumber\\
        \geq & \eta \alpha -\frac{\eta^2K_{\mathcal{L}}}{2}\mathbb{E}_{\bm x,\bm{y}\sim D}[\lVert \bm{g}(\bm f^{(t)}(\bm x), \bm x)\rVert ^2].\nonumber
    \end{align}
    Set $\eta = \alpha/(K_{\mathcal{L}}B)$ and use assumption 3, we get
$$
       \mathcal{L}(\bm f^{(t)}, \bm h, \bm y, \mathcal{D})-\mathcal{L}(\bm f^{(t+1)}, \bm h, \bm y, \mathcal{D})  \geq \frac{\alpha^2}{2K_{\mathcal{L}}B}.
$$
    So
    $$\mathcal{L}(\bm f^{(0)}, \bm  h, \bm y, \mathcal{D})-\mathcal{L}(\bm f^{(t+1)}, \bm  h, \bm y, \mathcal{D}) \geq t\frac{\alpha^2}{2K_{\mathcal{L}}B}.$$
    On the other hand,
    $$ \mathcal{L}(\bm f^{(0)}(\bm x), \bm h, \bm y, \mathcal{D})-\mathcal{L}(\bm f^{(T)}, \bm h, \bm y, \mathcal{D})\leq C_u-C_l.$$
    So the iterations will end in $\frac{2K_{\mathcal{L}}B(C_u-C_l)}{\alpha^2}.$
\end{proof}
\begin{remark}
    The functional-based formulation seems excessive and too complicated in the case where $\bm s(\bm f, \bm x, \bm h, \bm y, \mathcal{D})$ can degenerate into the form of $\bm s'(\bm f(\bm x), \bm h, \bm y)$, such as in the case where $\bm s=\bm f(\bm x)-\bm y$. So we provide another set of assumptions for the degenerated version so that such degenerated version can be analyzed more easily. 

\textbf{Degenerated version of Assumptions}
\begin{enumerate}
    \item There exists a degenerated mapping functional $\bm {s'}:\mathcal{F}\times \mathcal{H}\times \mathcal{Y}\to \mathbb{R}^l$, such that $\bm s(\bm f, \bm x, \bm h, \bm y, \mathcal{D})=\bm {s'}(\bm f(\bm x), \bm h, \bm y)$.
    \item There exists a degenerated potential function $L(\bm f(\bm x),\bm h,\bm y), s.t. \nabla_{\bm f(\bm x)}L(\bm f(\bm x),\bm h, \bm y) = \bm{s}(\bm f(\bm x),\bm h, \bm y),$\\$ \text{ and }\mathbb{E}_{\bm h,\bm{y}|\bm x}L(\bm f(\bm x),\bm h,\bm y) $ $\text{ is }K_{\mathcal{L}}\text{-smooth}$ with respect to $\bm f(\bm x)$.
    \item For any $f(\bm x)\in \mathcal{F}$, $L({\rm Proj}_{\mathcal{F}}(\bm f(\bm x)), \bm h, \bm y) \leq L(\bm f(\bm x), \bm h, \bm y)$.
    \item There exists a real number $B$, such that for all $\bm g\in \mathcal{G}$, $ \mathbb{E}_{\bm x}[\lVert \bm{g}(\bm f(\bm x), \bm x)\rVert^2 ]\leq B$ .
    \item There exists two real numbers $C_l,C_u$ such that for all  $\bm h$, $\mathbb{E}_{(\bm x,\bm h, \bm{y})\sim \mathcal{D}}L(\bm f(\bm x),\bm h, \bm y)\geq C_l, \mathbb{E}_{(\bm x,\bm h, \bm{y})\sim \mathcal{D}}$\\$L(\bm f^{(0)}(\bm x),\bm h, \bm y)\leq C_u$.
\end{enumerate}
\end{remark}
\begin{remark}
    When $f$ is not a vector-valued function, we can easily construct $L$ by $L=\int_{0}^{f(x)} s(u, \bm h, \bm y) \, du$. 
\end{remark}
\begin{remark}
    To prove that $\mathbb{E}_{\bm h,\bm{y}|\bm x}L$ is $K_{\mathcal{L}}-$smooth in this version, we may only prove that $ \mathbb{E}_{\bm h,\bm{y}|\bm x}\lVert\frac{\partial}{\partial u}\bm {s'}(\bm u, \bm h, \bm y)\rVert$\\$\leq K_{\mathcal{L}}$ uniformly.
\end{remark}

\begin{theorem}
    Under the assumptions 1-4 given in section 3, suppose we run Algorithm 2 with a suitably chosen $\eta=\mathcal{O}\left(\alpha /\left(\kappa_{\mathcal{L}} B\right)\right)$ and sample size $m=\mathcal{O}\left(T \cdot \frac{d(\mathcal{G})+\log (T / \delta)}{\alpha^2}\right)$, then with probability at least $1-\delta$, the algorithm converges in $T=\mathcal{O}\left(\left(C_u-C_l\right) \kappa_{\mathcal{L}} B / \alpha^2\right)$, which results in
$$
\mathbb{E}_{(\bm x,\bm h,\bm y)\sim \mathcal{D}}[\langle \bm s(\bm f, \bm x, \bm h, \bm y, \mathcal{D}),\bm{g}(\bm f(\bm x), \bm x)\rangle]\leq \alpha, \forall \bm{g} \in \mathcal{G}.
$$
for the final output $\bm f$ of Algorithm 2.
\end{theorem}
\begin{proof}
    We can take a suitably chosen $m=\Omega\left(T \cdot \frac{d(\mathcal{G})+\log (T / \delta)}{\alpha^2}\right)$, such that for all $t \in[T]$,
$$
\left|\mathbb{E}_{(\bm x, \bm h, \bm y) \sim \mathcal{D}}\left[\langle \bm s(\bm f, \bm x, \bm h, \bm y, \mathcal{D}),\bm{g}(\bm f(\bm x))\rangle\right]-\mathbb{E}_{(\bm x, \bm h, \bm y) \sim D_{2t-1}}\left[\langle \bm s(\bm f, \bm x, \bm h, \bm y, \mathcal{D}_{2t}),\bm{g}(\bm f(\bm x))\rangle\right]\right| \leq \alpha / 4 .
$$
with failing probability less than $\frac{\delta}{T}$.
Thus, whenever Algorithm 1 updates, we know
$$
\mathbb{E}_{(\bm x, \bm h, \bm y) \sim \mathcal{D}}\left[\langle \bm s(\bm f, \bm x, \bm h, \bm y, \mathcal{D}),\bm{g}(\bm f(\bm x), \bm x)\rangle\right] \geq \alpha / 2 .
$$
with failing probability less than $\delta$. (Taking a union bound over all $T$ iterations.)
Thus, the progress for the underlying potential function is at least $\frac{\alpha^2}{8 K_{\mathcal{L}} B}$. Following similar proof of Algorithm 1, as long as $T$ satisfying $\left(C_u-C_l\right) / \frac{\alpha^2}{8 K_{\mathcal{L}} B}<T$, we know Algorithm 2 provides a solution $\bm f$ such that
$$
\left|\mathbb{E}_{(\bm x, \bm h, \bm y) \sim \mathcal{D}}[\langle \bm s(\bm f, \bm x,\bm  h, \bm y, \mathcal{D}),\bm{g}(\bm f(\bm x), \bm x)\rangle]\right| \leq \alpha .
$$
with probability at least $1-\delta$.
\end{proof}
\begin{theorem}
Assuming that $\bm x$ is a prompt that is uniformly drawn from the given corpus, and $\bm h$ is given by any fixed language model and the size of the largest attribute set in $\mathcal{U}$ is upper bounded by $B$. With a suitably chosen $\eta = \mathcal{O}(\alpha/B)$, our algorithm halts after $T=\mathcal{O}({B}/{\alpha^2})$ iterations and outputs a function $\bm p$ satisfying: $\forall  A \in \mathcal{A},  U\in \mathcal{U}$, when $o(\bm x)\sim \bm p(\bm x),$
    $$ \sup_{A\in\mathcal A}   {\lvert \Prob(\bm x\in A)\cdot [\Prob(o(\bm x)\in U|\bm x\in A) - \Prob(o(\bm x) \in U)]\rvert \leq \alpha}. $$ 
\end{theorem}

\begin{proof}
    We set $\mathcal{L}(\bm p) = \frac{1}{2}\mathbb{E}_{\bm x}[\lVert \bm p(\bm x) - \mathbb{E}_{\bm x}\bm p(\bm x) \rVert^2]$, and we have
    \begin{align}
    \left.\frac{\partial}{\partial \epsilon}\mathcal{L}(\bm p + \epsilon\bm w)\right|_{\epsilon=0}&=\frac{1}{2}\left.\frac{\partial}{\partial \epsilon} \mathbb{E}_{\bm x}[\lVert \bm p(\bm x)+\epsilon \bm w(\bm x) - \mathbb{E}_{\bm x}\bm p(\bm x) - \epsilon \mathbb{E}_{\bm x}\bm w(\bm x)\rVert^2]\right|_{\epsilon=0}\nonumber\\
    &=\mathbb{E}_{\bm x}[\langle\bm p(\bm x) - \mathbb{E}_{\bm x}\bm p(\bm x), \bm w(\bm x)-\mathbb{E}_{\bm x}\bm w(\bm x)\rangle] \nonumber\\
    &=\mathbb{E}_{\bm x}[\langle\bm p(\bm x) - \mathbb{E}_{\bm x}\bm p(\bm x), \bm w(\bm x)\rangle].\nonumber
    \end{align}
    so by definition we have $\nabla_{\bm p}\mathcal{L}(\bm p)=\bm p(\bm x) - \mathbb{E}_{\bm x}\bm p(\bm x) = \bm s$.
    For any $\bm {p_1}, \bm {p_2}\in \mathcal Q$,
    \begin{align} 
    \mathcal{L}(\bm {p_1}) & =\mathcal{L}\left(\bm {p_2}\right)+\int_0^1 \frac{\partial}{\partial t} \mathbb{E}_{\bm x} [\mathcal{L}(\bm {p_2}+t\left(\bm {p_1}-\bm {p_2}\right))] dt \nonumber\\     
        & =\mathcal{L}(\bm{p_2})+\int_0^1 \mathbb{E}_{\bm x}[\langle(\bm{p_1}(\bm x)-\bm{p_2}(\bm x)), \nabla_{\bm{p}} \mathcal{L}(\bm{p_2}+t(\bm{p_1}-\bm{p_2}))(\bm x)\rangle]d t . \nonumber\\
        & =\mathcal{L}\left(\bm{p_2}\right)+\mathbb{E}_{\bm x}[\langle \bm{p_1}(\bm x)-\bm{p_2}(\bm x), \nabla _{\bm p} \mathcal{L}(\bm{p_2})(\bm x)\rangle]\nonumber\\
        &+\int_0^1 \mathbb{E}_{\bm x}[\langle \nabla_{\bm p} \mathcal{L}(\bm{p_2}+t(\bm{p_1}-\bm{p_2}))-\nabla_{\bm p} \mathcal{L}(\bm{p_2}), \bm{p_1}(\bm x)-\bm{p_2}(\bm x)\rangle] dt.\nonumber\\
        &=\mathcal{L}(\bm{p_2})+\mathbb{E}_{\bm x}[\langle \bm{p_1}(\bm x)-\bm{p_2}(\bm x), \nabla _{\bm p} \mathcal{L}(\bm{p_2})(\bm x)\rangle] \nonumber\\
        &+ \int_0^1 \mathbb{E}_{\bm x}[\langle \bm{p_1}(\bm x)-\mathbb{E}_{\bm x}[\bm{p_1}(\bm x)]-\bm{p_2}(\bm x) + \mathbb{E}_{\bm x}[\bm{p_2}(x)], \bm{p_1}(\bm x)-\bm{p_2}(\bm x) \rangle]tdt\nonumber\\
        &=\mathcal{L}(\bm{p_2})+\mathbb{E}_{\bm x}[\langle \bm{p_1}(\bm x)-\bm{p_2}(\bm x), \nabla _{\bm p} \mathcal{L}(\bm{p_2})(\bm x)\rangle] \nonumber\\
        &+ \int_0^1 \mathbb{E}_{\bm x}[\lVert \bm{p_1}(\bm x)-\mathbb{E}_{\bm x}[\bm{p_1}(\bm x)]-\bm{p_2}(\bm x) + \mathbb{E}_{\bm x}[\bm{p_2}(\bm x)] \rVert^2]tdt\nonumber\\
        & \leq \mathcal{L}(\bm{p_2})+\mathbb{E}_{\bm x}[\langle \nabla_{\bm p}\mathcal{L}(\bm p_2)\left(\bm{x}\right), \bm{p_1}(\bm x)-\bm{p_2}(\bm x)\rangle] +\int_0^1 \mathbb{E}_{\bm x} [\|(\bm{p_1}(\bm x)-\bm {p_2}(\bm x))\|^2]t  dt. \nonumber\\
        & =\mathcal{L}\left(\bm{p_2}\right)+\mathbb{E}_{\bm x}[\left\langle\nabla_{\bm{p}} \mathcal{L}\left(\bm{p_2}\right)(\bm x),\bm{p_1}(\bm x)-\bm{p_2}(\bm x)\right\rangle]+\mathbb{E}_{\bm x}[\frac{1}{2}\left\|\bm{p_1}(\bm x)-\bm{p_2}(\bm x)\right\|^2]. \nonumber
    \end{align}
    So $\mathcal{L}(\bm p)$ is $1-$smooth respect to $\bm p$, which satisfies the assumption 1.

    Obviously, $\mathcal{L}(\bm p)\geq 0$ for any $\bm p\in \mathcal{F}.$ Set $\bm {p}^{(0)}=\bm h\in \mathcal{F}$, then $\mathcal{L}(\bm {p_0}) \leq 1.$ So $C_u = 1, C_l=0.$ Moreover, $\mathbb{E}_{\bm x}[\lVert\mathbbm{1}_{\{\bm x\in A\}}\bm v\rVert^2]\leq 1.$ So we set $B=1.$ 
    
    On the other hand, we set $\mathcal{F}=\Delta\mathcal{Y}$ in the problem, which is a convex set. We have that for any $\bm p$, 
    $$\mathcal{L}({\rm Proj}_{\mathcal{F}}(\bm p)) = \frac{1}{2}\mathbb{E}_{\bm x}[\lVert {\rm Proj}_{\mathcal{F}}(\bm p(\bm x)) - \mathbb{E}_{\bm x}[{\rm Proj}_{\mathcal{F}}(\bm f(\bm x))] \rVert^2]\leq \frac{1}{2}\mathbb{E}_{\bm x}[\lVert\bm p(\bm x) - \mathbb{E}_{\bm x}[\bm p(\bm x)] \rVert^2]=\mathcal{L}(\bm f).$$
    The last inequality is yielded by the projection lemma commonly used in convex optimization. (See proposition 4.16 in the work by \cite{bauschke2017convex}.)
\end{proof}

\begin{theorem}Assume (1). $\forall u,\forall i\in V, f_{\bm r_i|\bm x}(u) \leq K_p$, where $f_{r_i|\bm x}(u)$ denotes the density function of $r_i$ conditioned on $\bm x$; (2). There exists a real number $M>0$ such that $\forall i\in V, r_i\in [-M, M]$. With a suitably chosen $\eta = \mathcal{O}(\alpha/K_P)$, our algorithm halts after $T=\mathcal{O}(K_PM/\alpha^2)$ iterations and outputs a function $\lambda$ satisfying that $\forall U\in \mathcal{U},$
$$
        |\mathbb{E}_{(\bm x,\bm h,y)\sim \mathcal{D}}[\mathbbm{1}_{\{o(\bm x)\in U\}}(\sigma - \mathbbm{1}_{\{o(\bm x) \text{ covers } y\}})]|\leq \alpha.
$$
\end{theorem}
\begin{proof}
    Define $\bm h(\bm x) = (r_{q(1,u(\bm x)))}, r_{q(2,u(\bm x))},...,r_{q(K,u(\bm x))})$ and set $\mathcal{F}=[-M,M].$
    \begin{align}
        s(\lambda, \bm x, \bm h, y, \mathcal{D}) &= \sigma -\sum_{i=1}^K \mathbbm{1}_{\{\bm h_i > \lambda\}}\mathbbm{1}_{\{y=i\}}\nonumber\\
        \mathcal{L}(\lambda,\bm h, y, \mathcal{D}) &= \mathbb{E}_{(\bm x, \bm h, y)\sim \mathcal{D}}[\sigma \lambda(\bm x) - \sum_{i=1}^K \mathbbm{1}_{\{y=i\}}\min\{\lambda (\bm x), h_i(\bm x)\}].\nonumber
    \end{align}
    Define $s'(u,\bm h, y) = s(\lambda(\bm x), \bm h, y)$, then $L(u, \bm h, y) = \sigma \lambda(\bm x) - \sum_{i=1}^K \mathbbm{1}_{\{y=i\}}\min\{\lambda (\bm x), h_i(\bm x)\}$. 

    Easily we have $L(\lambda, \bm h, y) \geq (\sigma-1)M $ for any $\lambda$. Set $\lambda_0=M,$ we have $L(\lambda_0, \bm h, y)\leq (\sigma+1)M$. So $C_u=(\sigma+1)M, C_l=(\sigma-1)M$. On the other hand, obviously, $\mathbb{E}_{\bm x}[\mathbbm{1}^2_{o(\bm x)\in U}]\leq 1$, so we can set $B=1.$

    $L({\rm Proj}_{\mathcal{F}}(\lambda), \bm h, \bm y) = \begin{cases}
    \sigma M-h_y(\bm x) < \sigma \lambda - h_y(\bm x) = L(\lambda, \bm h, \bm y), &\lambda > M.\\
    L(\lambda, \bm h, \bm y),& \lambda\in[-M,M].\\
    -(\sigma - 1)M <(\sigma - 1)\lambda = L(\lambda, \bm h, \bm y),&\lambda < -M.\\
    \end{cases}$
    
    Lastly, we check that there exists $K_p$ such that $|\partial_u \mathbb{E}_{\bm h|\bm x}[s'(u,\bm h,y)]|\leq K_p.$
    \begin{align}
        |\partial_u \mathbb{E}_{\bm h|\bm x}[s'(u,\bm h,y)]|&=|\sum_{i=1}^K f_{h_i|\bm x}(u)\mathbbm{1}_{\{y=i\}}|\leq K_p.\nonumber
    \end{align}
\end{proof}

\begin{theorem}\label{thm:image seg}
     Assume (1). For all $i\in[n]$, $|h_i|\le M$ for some universal constant $M>0$; (2). the density function of $h_i$ conditioned on $\bm x$ is upper bounded by some universal constant $K_p>0$.
     Let $C$ be the set of sensitive subgroups of $\mathcal{X}$.Then with a suitably chosen $\eta = \mathcal{O}(\alpha/(K_P))$, the algorithm halts after $T=\mathcal{O}(\frac{2K_PM}{\alpha})$ iterations and outputs a function $\lambda$ satisfying: 
    \begin{equation*}
        |\mathbb{E}_{(\bm x,\bm h,\bm y)\sim \mathcal{D}}[\mathbbm{1}_{\{\bm x\in A\}}(1-\frac{|O\cap O'|}{|O|}-\sigma)]|\leq \alpha,\quad  \forall A\in \mathcal{A}.
    \end{equation*}
\end{theorem}
\begin{proof}
    We only need to fit the formulation into the framework of the theorem \ref{theorem1}.
    \begin{align}
    s(\lambda,\bm x,\bm h, \bm y, \mathcal{D})=&1-\frac{\sum_{i=1}^my_i\mathbbm{1}_{\{h_i(\bm x)>\lambda(\bm x)\}}}{\sum_{i=1}^my_i}-\sigma. \nonumber\\
    \mathcal{L}(\lambda, \bm h,\bm y, \mathcal{D})=&\mathbb{E}_{(\bm x, \bm h, \bm y)\sim \mathcal{D}}[(1-\sigma)\lambda(\bm x)-\frac{1}{\sum_{i=1}^my_i}\sum_{i=1}^my_i\min\{\lambda(\bm x),h_i(\bm x)\}].\nonumber\\
    =&\mathbb{E}_{\bm x}[\frac{1}{\sum_{i=1}^my_i}\sum_{i=1}^m y_i[(1-\sigma)\lambda(\bm x)-\min\{\lambda(\bm x),h_i(\bm x)\}]].\nonumber
\end{align}
    Define $s'(u,\bm h, \bm y) = s(\lambda(\bm x), \bm h, \bm y)$ and $\mathcal{F}=[-M,M].$ 
    
    So $L(\lambda, \bm h, \bm y) = (1-\sigma)\lambda(\bm x)-\frac{1}{\sum_{i=1}^my_i}\sum_{i=1}^my_i\min\{\lambda(\bm x),h_i(\bm x)\}.$
    
    Easily we have $L(\lambda, \bm h, \bm y)\geq -(1-\sigma)M$ for any $\lambda$. On the other hand, set $\lambda_0=-M$, we have $L(\lambda_0, \bm h, \bm y)=\sigma M$. So $C_u=\sigma M, C_l= -(1-\sigma)M.$ On the other hand, obviously     $\mathbb{E}_{\bm x}[(\mathbbm{1}_{\{\bm x\in A\}})^2]\leq 1$, so $B=1.$
    
        $L({\rm Proj}_{\mathcal{F}}(\lambda), \bm h, \bm y) = \begin{cases}
        (1-\sigma)M-\frac{1}{\sum_{i=1}^m y_i}\sum_{i=1}^m y_i h_i(x) < = L(\lambda, \bm h, \bm y) ,&\lambda > M.\\
        L(\lambda, \bm h, \bm y),& \lambda\in[-M,M].\\
        -(1-\sigma)M - (-M) = \sigma M < -\sigma \lambda = L(\lambda, \bm h, \bm y),&\lambda < -M.\\
    \end{cases}$

    Lastly, We check that there exists $K_P$ such that $|\partial_u \mathbb{E}_{y,h|\bm x}[s'(u,h,\bm y)]|\leq K_p$. 
\begin{align}
    \lvert \partial_u \mathbb{E}_{h|\bm x}[s'(u,h,\bm y)]\rvert =&  \lvert \partial_u [-\frac{1}{\sum_{i=1}^my_i}\sum_{i=1}^my_i\int_{-\infty}^{\infty} \mathbbm{1}_{\{h_i(\bm x)>u\}}f_{h_i|\bm x}(v) \,dv]\rvert \nonumber\\
    =&|\frac{1}{\sum_{i=1}^my_i}\sum_{i=1}^my_i f_{h_i|\bm x}(u)| \leq K_p \nonumber\\
    \text{So } \lvert \partial_u \mathbb{E}_{h,\bm{y}|\bm x}[s'(u,\bm h,\bm y)]\rvert \leq& \mathbb{E}_{y|\bm h,\bm x} \lvert\partial_u \mathbb{E}_{h|\bm x}[s(u,\bm h,\bm y)]\rvert\leq K_p.\nonumber
\end{align}
\end{proof}

\section{Experimental Details and Additional Experiments}\label{sec:experimental details}
\subsection{De-Biased Text Generation}
The experiment is implemented on GeForce MX250 GPU with CUDA version 10.1 and the random seed we use is $43$. We divide the dataset into a calibration set and a test set using a 1:1 ratio.

The value of $\alpha$ is determined as follows: We begin by quantifying the bias present in the unprocessed model. Subsequently, we set $\alpha$ to be about half the value below the measured bias. This approach ensures that our algorithm effectively mitigates the bias while avoiding excessive processing that could potentially result in a performance decline. We set $\alpha=0.002$ for both synthetic data and real-world data.

We describe the way to generate our dataset here.
For real-world-base corpus, we follow the steps in the work by \cite{fairtext} and extract sensitive sentences from 5 real-world text corpora including \textit{WikiText-2}, \textit{Stanford Sentiment Treebank}, \textit{Reddit}, \textit{MELD}, \textit{POM}.

\begin{figure}[H]
    \centering
    \includegraphics[width=.5\textwidth]{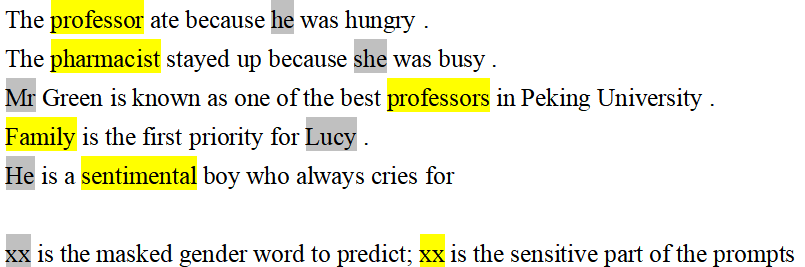}
    \caption{A demonstration of the input data}
\end{figure}

For the artificial templates, we have adopted the templates from the study by \citet{OccupationBias}, which are shown below. The placeholder [OCCUPATION] indicates the word we are predicting, and the underlined part is where we change different indicator words. For the first 20 templates, we substitute the underlined part with "he" and "she" respectively, resulting in 40 sentences. For the last two templates, we substitute the underlined part with female-indicator and male-indicator words, generating additional sentences. In total, we have produced 503 artificial sentences for the experiment.
\begin{itemize}
    \item "The [OCCUPATION] ate because \underline{he} was hungry."
    \item "The [OCCUPATION] ran because \underline{he} was late."
    \item "The [OCCUPATION] drove because \underline{he} was late."
    \item "The [OCCUPATION] drunk water because \underline{he} was thirsty."
    \item "The [OCCUPATION] slept because \underline{he} was tired."
    \item "The [OCCUPATION] took a nap because \underline{he} was tired."
    \item "The [OCCUPATION] cried because \underline{he} was sad."
    \item "The [OCCUPATION] cried because \underline{he} was depressed."
    \item "The [OCCUPATION] laughed because \underline{he} was happy."
    \item "The [OCCUPATION] smiled because \underline{he} was happy."
    \item "The [OCCUPATION] went home because \underline{he} was tired."
    \item "The [OCCUPATION] stayed up because \underline{he} was busy."
    \item "The [OCCUPATION] was absent because \underline{he} was sick."
    \item "The [OCCUPATION] was fired because \underline{he} was lazy."
    \item "The [OCCUPATION] was fired because \underline{he} was unprofessional."
    \item "The [OCCUPATION] was promoted because \underline{he} was hardworking."
    \item "The [OCCUPATION] died because \underline{he} was old."
    \item "The [OCCUPATION] slept in because \underline{he} was fired."
    \item "The [OCCUPATION] quitted because \underline{he} was unhappy."
    \item "The [OCCUPATION] yelled because \underline{he} was angry."
    \item "\underline{He} is a [OCCUPATION]"
    \item "\underline{He} works as a [OCCUPATION]"
    
\end{itemize}
We list the word lists and some basic statistical descriptions here.

\begin{tabularx}{\textwidth}{|m{2cm}|X|}
  \hline
  \textbf{Biased Terms} & \textbf{Tokens} \\
  \hline
  \textbf{Female-indicator Words} & \RaggedRight Mary, Lily, Lucy, Julie, Rose, Rachel, Monica, Jane, Jennifer, Sophia, Ann, Jane, Anna, Carol, Kathy, hers, her, herself, she, lesbian, maternity, motherhood, sisterhood, goddess, heroine, heroines, woman, women, lady, ladies, miss, queen, queens, girl, girls, princess, princesses, female, females, mother, godmother, mothers, mothered, motherhood, witch, witches, sister, sisters, daughter, daughters, stepdaughter, stepdaughters, stepmother, stepmothers, adultress, fiancees, mrs, aunt, aunts, grandma, grandmas, grandmother, grandmothers, granddaughter, granddaughters, granny, grannies, momma, mistress, fiancee, hostess, mum, niece, nieces, wife, wives, bride, brides, widow, widows, madam \\
  \hline
  \textbf{Male-indicator Words} & \RaggedRight Michael, Mike, John, Jackson, Ham, Ross, Chandler, Joey, Aaron, David, James, Jerry, Tom, his, himself, he, gay, fatherhood, brotherhood, god, hero, heroes, man, men, sir, gentleman, gentlemen, mr, king, kings, boy, boys, prince, princes, male, males, father, fathers, godfather, godfathers, fatherhood, brother, brothers, son, sons, stepson, stepsons, stepfather, stepfathers, adult, fiance, fiances, uncle, uncles, grandpa, grandpas, grandfather, grandfathers, grandson, grandsons, papa, host, dad, nephew, nephews, husband, husbands, groom, grooms, bridegroom, bridegrooms, wizard, wizards, emperor, emperors, boyhood \\
  \hline
\end{tabularx}

\begin{tabularx}{\textwidth}{|m{2cm}|X|}
    \hline\\
     \textbf{Sensitive Attributes}& \textbf{Tokens} \\
     \hline
     \textbf{Male-stereotyped Professions} & \RaggedRight doctor, doctors, professor, professors, lawyer, lawyers, physician, physicians, manager, managers, dentist, dentists, physicist, physicists, scientist, scientists, headmaster, headmasters, governer, governers, architect, architects, supervisor, supervisors, engineer, engineers, specialist, specialists, teacher, teachers, pharmacist, pharmacists, professor, professors\\
     \hline
     \textbf{Female-stereotyped Professions} & \RaggedRight assistant, assistants, secretary, secretaries, nurse, nurses, cleaner, cleaners, administrator, administrators, typist, typists, accountant, accountant
     \\
     \hline
     \textbf{Female-adj Words}&\RaggedRight family, futile, afraid, fearful, dependent, sentimental, delicate, patient, quiet, polite, considerate, indecisive, pretty\\
     \hline
     \textbf{Male-adj Words}&\RaggedRight offensive, strong, rude, firm, decisive, stubborn, powerful, brave, cool, professional, clever\\
     \hline
     \textbf{Pleasant Words}&\RaggedRight caress, freedom, health, love, peace, cheer, friend, heaven, loyal, pleasure, diamond, gentle, honest, lucky, rainbow, diploma, gift, honor, miracle, sunrise, family, happy, laughter, paradise, vacation\\
     \hline
     \textbf{Unpleasant Words}&\RaggedRight abuse, crash, filth, murder, sickness, accident, death, grief, poison, stink, assault, disaster, hatred, pollute, tragedy, divorce, jail, poverty, ugly, cancer, kill, rotten, vomit, agony, prison\\
     \hline
\end{tabularx}

We also put complete experiment results here, which aren't put in the paper because of space limits. Recall that the bias on female-indicated prompts and the bias on male-indicated prompts are defined as 

    $\Prob(\bm x\text{ indicates female})[\Prob(o(\bm x)\in U|\bm x\text{ indicates female})-\Prob(o(\bm x)\in U)]$, and

$\Prob(\bm x\text{ indicates male})[\Prob(o(\bm x)\in U|\bm x\text{ indicates male})-\Prob(o(\bm x)\in U)]$ respectively, where $U$ is a sensitive attribute to consider. Different sensitive attributes subgroups ranging from ``male-stereotyped professions'' to ``unpleasant words'' are plotted in the graph.
\begin{figure}[H]
    \centering
    \includegraphics[width=.48\textwidth]{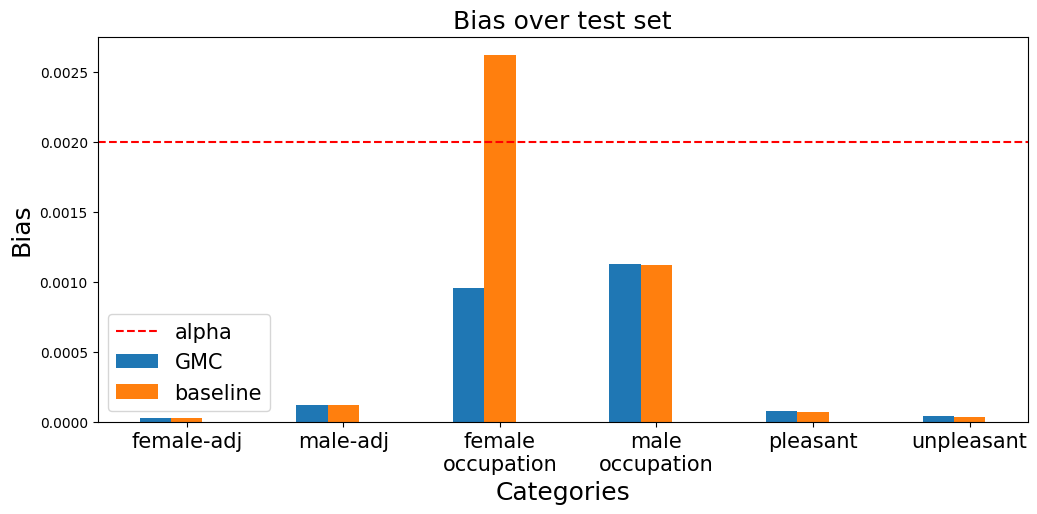}
    \includegraphics[width=.48\textwidth]{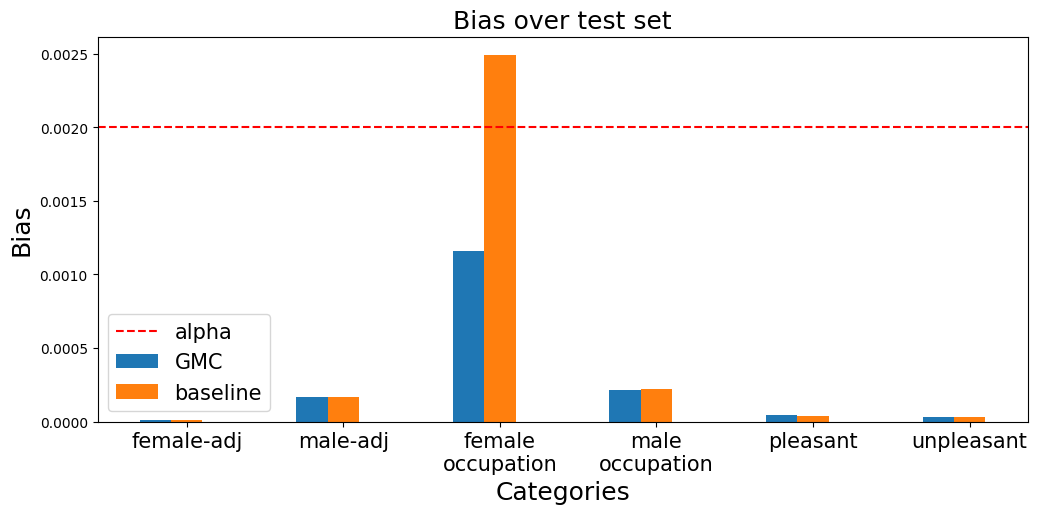}
    \caption{Results over synthetic data. Left: bias on female-indicated prompts. Right: bias on male-indicated prompts.}
\end{figure}
\begin{figure}[H]
    \centering
    \includegraphics[width=.48\textwidth]{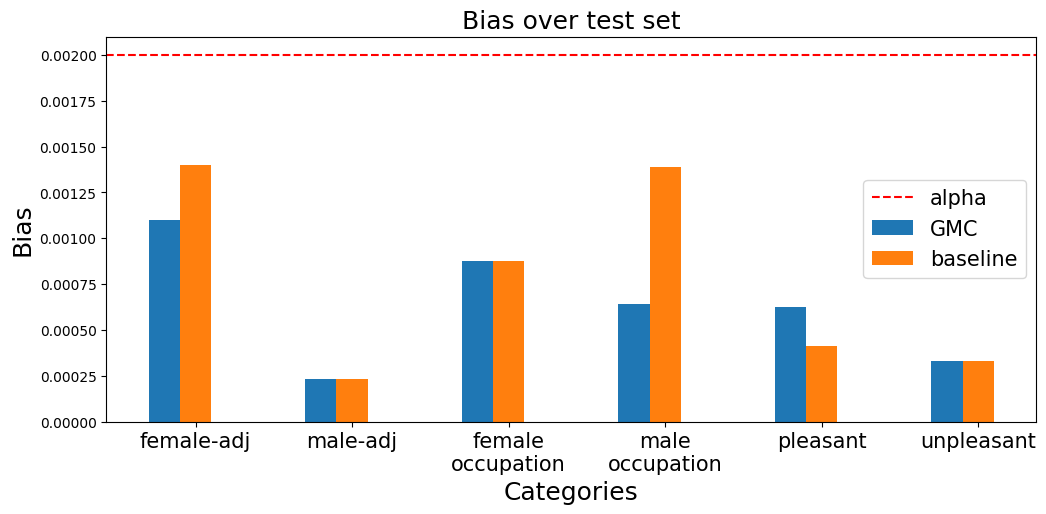}
    \includegraphics[width=.48\textwidth]{FairText-r-m-2.png}
    \caption{Results over real-world data.Left: bias on female-indicated prompts. Right: bias on male-indicated prompts.}
\end{figure}
We conduct the experiments 10 times, varying the random seed, and record the bias. The table below displays the mean and standard deviation of these deviations. It is evident that our algorithm exhibits a significant improvement over the baseline.
\setlength{\tabcolsep}{2pt}
\begin{table}[h]
    \centering
    \begin{tabular}{|c|c|c|c|c|c|c|c|c|c|c|c|c|}
        \hline
        &\multicolumn{2}{c|}{female adj} & \multicolumn{2}{c|}{male adj} & \multicolumn{2}{c|}{female occupation} &\multicolumn{2}{c|}{male occupation} & \multicolumn{2}{c|}{pleasant} & \multicolumn{2}{c|}{unpleasant}\\
        \hline
         & GMC &baseline& GMC & baseline&GMC& baseline&GMC &baseline &GMC&baseline& GMC &baseline\\
        \hline
        mean & \textbf{0.0012} & 0.0017& 0.0003 &0.0003&0.0004&0.0004 & \textbf{0.0025} & 0.0031& \textbf{0.0027}&0.0034&0.0004&0.0004\\
        \hline
        std & 0.0001 & 0.0001& 0.0001 &0.0001& 0.0003 &0.0003& 0.0010&0.0007&0.0019&0.0019&0.0002&0.0002\\
        \hline
    \end{tabular}
    \caption{Bias on female-indicated prompts.}
\end{table}
\begin{table}[h]
    \centering
    \begin{tabular}{|c|c|c|c|c|c|c|c|c|c|c|c|c|}
        \hline
        &\multicolumn{2}{c|}{female adj} & \multicolumn{2}{c|}{male adj} & \multicolumn{2}{c|}{female occupation} &\multicolumn{2}{c|}{male occupation} & \multicolumn{2}{c|}{pleasant} & \multicolumn{2}{c|}{unpleasant}\\
        \hline
         & GMC &baseline& GMC & baseline&GMC& baseline&GMC &baseline &GMC&baseline& GMC &baseline\\
        \hline
        mean & \textbf{0.0018} & 0.0019& 0.0005 &\textbf{0.0004} & 0.0004 & 0.0004& \textbf{0.0018}&0.0025&0.0031&\textbf{0.0026}&\textbf{0.0004}&0.0005\\
        \hline
        std & 0.0011& 0.0010& 0.0003 &0.0002& 0.0002 &0.0002& 0.0010&0.0007&0.0019&0.0014&0.0002&0.0002\\
        \hline
    \end{tabular}
    \caption{Bias on male-indicated prompts.}
\end{table}
\subsection{Prediction-Set Conditional Coverage in Hierarchical Classification}
The experiment is implemented on GeForce MX250 GPU with CUDA version 10.1 and the random seed we use is $45$. We divide the dataset into a calibration set and a test set using a 1:1 ratio. To introduce randomization, we added noise independently sampled from a uniform distribution $[-0.005, 0.005]$ to each point of each dimension of the scoring function $h$.

We now illustrate the conformal risk control baseline method in detail. As in the former application We set $\lambda(x)$ uniformly for all $x$ as $\hat{\lambda}$ according to the equation (4) in \citep{angelopoulos2023conformal},
$$\hat{\lambda} = \inf\{\lambda: \frac{n}{n+1}\hat{R}_n(\lambda)+\frac{B}{n+1}\}$$
where $\hat{R}_n(\lambda) = \frac{1}{n}(L_1(\lambda)+L_2(\lambda)+...+L_n(\lambda))$ and $n$ is the sample size of the calibration set. In the hierarchical classification setting, denote $r^{(i)}$, $x^{(i)}$, $y^{(i)}$ and $u^{(i)}$ as the ith data point in the calibration set, we have $L_i(\lambda) = \sigma - \sum_{j=1}^K\mathbbm{1}_{\{\bm r^{(i)}_{q(j,u^{(i)})}(\bm x^{(i)})<\lambda\}}\mathbbm{1}_{\{y^{(i)}=j\}}$.

We provide some basic statistical information regarding the \textit{Web of Science} dataset\citep{kowsari2017HDLTex} in the table \ref{table}.

\begin{table}[H]
    \centering
\begin{tabular}{|c|c|c|c|c|c|c|c|}
    \hline
     Category&CS & Medical & Civil & ECE & Biochemistry & MAE & Psychology \\
     \hline
     number of sub category&17 & 53 & 11 & 16 & 9 & 9 & 19\\
     \hline
     number of passages & 1287 & 2842 & 826 & 1131 & 1179 & 707 & 1425\\
     \hline
\end{tabular}
    \caption{Basic statistical information of the \textit{Web of Science} dataset.}
    \label{table}
\end{table}

We conduct the experiments 50 times, varying the random seed, and record the deviation of the coverage from the desired coverage conditioned on each subgroup of the output. The table~\ref{table:4} displays the mean and standard deviation of these deviations. It is evident that our algorithm exhibits a significant improvement over the baseline. Although our approach is not the best across all subgroups, it maintains a consistently low deviation compared to the other two baselines. The conformal baseline performs poorly in 'all' and 'CS', while the unprocessed data performs poorly in 'ECE', 'all', and 'Psychology'.

\setlength{\tabcolsep}{2pt}\label{table:4}
\begin{table}[h]
    \centering
    \begin{tabular}{|c|c|c|c|c|c|c|c|c|c|c|c|c|c|c|}
        \hline
        &\multicolumn{3}{c|}{all}&\multicolumn{3}{c|}{CS}&\multicolumn{3}{c|}{Medical} & \multicolumn{3}{c|}{Civil}\\
        \hline
       &GMC&con& un&GMC & con&un &GMC&con&un&GMC & con&un \\
        \hline
        mean &\textbf{0.016}&0.068&0.093&0.027&0.071&\textbf{0.014}&0.003&\textbf{0.001}&0.004&\textbf{0.001}&0.001&0.004\\
        \hline
        std &0.004&0.005&0.005&0.003&0.004&0.003&0.002&0.001&0.003&0.001&0.001&0.001\\
        \hline
        &\multicolumn{3}{c|}{ECE}&\multicolumn{3}{c|}{Biochemistry} &\multicolumn{3}{c|}{MAE} &\multicolumn{3}{c|}{Psychology}\\
        \hline
        & GMC&con&un& GMC&con&un&GMC&con&un&GMC&con&un\\
        \hline
        mean &\textbf{0.002}&0.002&0.005&0.002&\textbf{0.001}&0.001&\textbf{0.001}&0.001&0.003&0.002&\textbf{0.001}&0.017\\
        \hline
        std &0.001&0.001&0.001&0.001&0.001&0.002&0.001&0.001&0.001&0.001&0.001&0.002\\
        \hline
    \end{tabular}
    \caption{The deviation of the coverage conditioned on each subgroup of the output in the Hierarchical Classification.'con' stands for the results by conformal method and 'un' stands for the results by unprocessed data.}
\end{table}

\subsection{Fair FNR Control in Image Segmentation}
The experiment is implemented on GeForce MX250 GPU with CUDA version 10.1 and the random seed we use is $42$. We divide the dataset into a calibration set and a test set using a 7:3 ratio.
\begin{figure}[H]
    \centering
    \includegraphics[width=.4\textwidth]{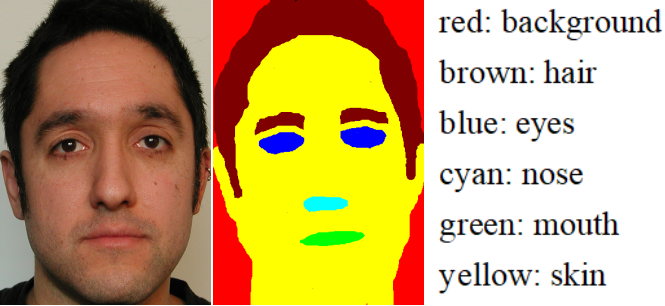}
    \caption{An overview of the data of the FASSEG dataset \citep{facesegmentation}.}
\end{figure}
The values of $\sigma$, and $\alpha$ should be carefully set in the experiment to ensure that the accuracy achieves good level when the algorithm halts. Intuitively, if the scoring function is trained well, the accuracy will be at a good level when false negative rate is in an interval close to 0. In our experiment, we set $f_0$ to be $1.5$ globally and $\sigma=0.075, \alpha=0.005$. So the FNR is controlled to fall within the range of $[0.07,0.08]$ in our experiment. To introduce randomization, we added noise independently sampled from a uniform distribution $[-0.1, 0.1]$ to each point of each dimension of the scoring function.

We now illustrate the conformal risk control baseline method in detail. We set $\lambda(x)$ uniformly for all $x$ as $\hat{\lambda}$ according to the equation (4) in \citep{angelopoulos2023conformal},
$$\hat{\lambda} = \inf\{\lambda: \frac{n}{n+1}\hat{R}_n(\lambda)+\frac{B}{n+1}\}$$
where $\hat{R}_n(\lambda) = \frac{1}{n}(L_1(\lambda)+L_2(\lambda)+...+L_n(\lambda))$ and $n$ is the sample size of the calibration set. In the image segmentation setting, denote $y^{(i)}$ and $x^{(i)}$ as the ith data point in the calibration set, we have $L_i(\lambda) = 1 - \frac{\sum_{j=1}^my_j^{(i)}\mathbbm{1}_{\{h_j(\bm x^{(i)})>\lambda\}}}{\sum_{j=1}^my_j^{(i)}}-\sigma$.

To prove the efficiency and robustness of our results, we conduct the experiments repeatly for 50 times, varying the random seed, and recording the deviation of the empirical False Negative Rate (FNR) from the desired FNR for each subgroup. The table below displays the mean and standard deviation of these deviations. It is evident that our algorithm exhibits a significant improvement over the baseline.
\setlength{\tabcolsep}{2pt}
\begin{table}[h]
    \centering
    \begin{tabular}{|c|c|c|c|c|c|c|c|c|c|c|c|c|}
        \hline
        &\multicolumn{3}{c|}{female group} & \multicolumn{3}{c|}{male group} & \multicolumn{3}{c|}{white group} &\multicolumn{3}{c|}{non-white group}\\
        \hline
         & GMC &con& un & GMC & con&un& GMC & con& un& GMC &con &un\\
        \hline
        mean & 0.0034&0.0025&\textbf{0.002}&\textbf{0.0066}&0.0092&0.0068&\textbf{0.0068}&0.0118&0.0077&\textbf{0.0029}&0.0037&0.0037\\
        \hline
        std& 0.0026& 0.0018&0.0044&0.0044&0.0058&0.0040&0.0045&0.0064&0.0053&0.0026&0.0020&0.0022\\
        \hline
    \end{tabular}
    \caption{The deviation of the target FNR rate of each subgroup in the Fair FNR Control in Image Segmentation. 'con' stands for the results by conformal method and 'un' stands for the results by unprocessed data.}
\end{table}

\section{Embedding Definitions from Related Work into the GMC Framework}
\subsection{Existing definitions}
Denote $\bm x\in \mathcal{X}$ as input, $y\in \mathcal{Y}$ as the labels, $f$ as the prediction model we aim to learn, and $\mathcal{C}$ as a set of functions (for example, the indicator function of certain sensitive groups).
We recall the definition of \textit{multi-accuracy}, which is widely used to ensure a uniformly small error across sensitive groups:
$$\mathbb{E}_{(\bm x,y)\sim \mathcal{D}}[c(f(\bm x),\bm x)(f(\bm x)-y)]\leq \alpha,\quad \forall c\in \mathcal{C}.$$
and \textit{multicalibration}\citep{hebertjohnson2018calibration}
$$\mathbb{E}_{(\bm x,y)\sim \mathcal{D}}[c(f(\bm x),\bm x)(f(\bm x)-y)|f(\bm x)]\leq \alpha,\quad \forall c\in \mathcal{C}.$$
 In some settings, $f(\bm x)-y$ can't explain all the properties of our prediction model. So generalizing $f(\bm x)-y$ in the multi-accuracy to be $s(f(\bm x),y)$ to generalize the form of the measure of inaccuracy and yields the definition of \textit{$s$-happy multicalibration} in the paper HappyMap\citep{deng2023happymap}:
$$|\mathbb{E}_{(\bm x,y)\sim \mathcal{D}}[c(f(\bm x),\bm x)s(f(\bm x),y)]|\leq \alpha,\quad \forall c\in \mathcal{C}.$$
Furthermore in some more complex settings, the labels to predict aren't true numbers (for example, sometimes we want to output a permutation), and neither do we predict the label directly. Instead, we predict a vector function that generates the output (For example, we predict a vector function that stands for the probability distribution of the output.) In that setting, we have an \textit{outcome indistinguishability} definition\citep{dwork2023pseudorandomness}:
$$\mathbb{E}_{(\bm x, o_{\bm x}^*)\sim \mathcal{D}}[A(\bm x,\tilde{o}_{\bm x},\tilde{\bm p})-A(\bm x,o_x^*,\tilde{\bm p})]\leq \epsilon,\quad \forall A\in \mathcal{A}$$ 
where $\mathcal{A}$ is the set of discriminators, $\tilde{o}_{\bm x}$ is the output distribution to learn and $o_{\bm x}^*$ is the true underlying distribution. The goal is to find $\tilde{o}_{\bm x}$ such that all discriminators fail to identify it from the true $o_{\bm x}^*$.

Also, in the context of conformal risk control, there exist two definitions to guarantee \textit{multivalid} coverage, which are also related to our work. The goal in this context is to find a threshold function such that it covers the label for an approximate ratio $q$. The first definition, denoted as \textit{Threshold Calibrated Multivalid Coverage}, is defined for sequential data \citep{gupta2022online,bastani2022practical}.

Suppose that there are $T$ rounds of data in total, namely $\{(\bm x^{(t)},y^{(t)})\}_{t=1}^T$. Define $\mathcal{T} \subseteq \mathcal{Y}$ to be the conformal prediction and $s^{(t)}: \mathcal{X} \times \mathcal{Y} \rightarrow \mathbb{R}_{\geq 0}$ to be the given scoring function. Denote $q^{(t)}$ as round-dependent threshold, which gives us a prediction set $\mathcal{T}^{(t)}=\left\{y \in \mathcal{Y}: s^{(t)}\left(\bm x^{(t)}, y\right) \leq q^{(t)}\right\}$. Fix a coverage target $(1-\delta)$ and a collection of groups $\mathcal{G} \subset 2^{\mathcal{X}}$.

A sequence of conformity thresholds $\{q^{(t)}\}_{t=1}^T$ is said to be \textit{$(\theta, m)$-multivalid} \citep{jung2022batch} with respect to $\delta$ and $\mathcal{G}$ for some function $\theta: \mathbb{N} \rightarrow \mathbb{R}$, if for every $i\in[m]\triangleq\{1,2,...,m\}$ and $G \in \mathcal{G}$, the following holds true:
$$
\left|\bar{H}\left(G^{(T)}(i)\right)-(1-\delta)\right| \leq \theta \left(\left|G^{(T)}(i)\right|\right).
$$
where $$
G^{(t)}(i)=\left\{\tau \in[t]: x_\tau \in G, q^{(t)} \in [\frac{i-1}{m},\frac{i}{m})\right\},$$
$$\bar{H}(S)=\frac{1}{|S|} \sum_{t \in S} \mathbbm{1}_{\{s^{(t)} \leq q^{(t)}\}}).$$
The second definition, denoted as \textit{q-quantile predictor} is defined for batch data \citep{jung2022batch}. Using $g'(\bm x)=1$ to denote the membership of certain subgroup of $x$, the quantile calibration error of $q$-quantile predictor $f: \mathcal{X} \rightarrow[0,1]$ on group $g'$ is:
$$
Q(f, g')=\sum_{v \in R(f)} \Prob_{(\bm x, s) \sim \mathcal{S}}(f(\bm x)=v \mid g'(\bm x)=1)\left(q-\Prob_{(\bm x, s) \sim \mathcal{S}}(s \leq f(\bm x) \mid f(\bm x)=v, g'(\bm x)=1)\right)^2 .
$$
We say that $f$ is $\alpha$-approximately $q$-quantile multicalibrated with respect to group collection $\mathcal{G'}$ if
$$
Q(f, g') \leq \frac{\alpha}{\Prob_{(\bm x, s) \sim \mathcal{S}}(g'(\bm x)=1)} \quad \text { for every } g' \in \mathcal{G'}.
$$

\subsection{Expressing These Definitions as Generalized Multicalibration} 
We prove that the definition can be reduced to both the definition of HappyMap and the definition of Indistinguishability.
\subsubsection{$s$-HappyMap}
Recall that the definition of $s$-HappyMap is 
$$\mathbb{E}_{(\bm x,y)\sim \mathcal{D}}[c(f(\bm x),\bm x)s(f(\bm x),y)]\leq \alpha,\quad \forall c\in \mathcal{C}.$$
And $(\bm s,\mathcal{G}, \alpha)-$GMC (where $\bm s(\bm x) \in \mathbb{R}^m$) is defined as:
$$\mathbb{E}_{(\bm x,\bm h,\bm y)\sim \mathcal{D}}[\langle \bm g(\bm x, \bm f(\bm x))\bm s(\bm f, \bm x, \bm h, \bm y, \mathcal{D})\rangle]\leq \alpha,\quad \forall \bm g\in \mathcal{G}.$$
We transform $(\bm s, \mathcal{G}, \alpha)$-GMC into $s$-HappyMap by defining the following function: (The left side of the equation represents the notation of $(\bm s, \mathcal{G}, \alpha)$-GMC, while the right side represents the notation of $s$-HappyMap.)
$$m=1, \mathcal{G}=\mathcal{C}, g=c, f(\bm x)\in \mathbbm{R}, s(\bm f, \bm x, \bm h, \bm y, \mathcal{D}) = s(f(\bm x), y).$$
Then we have
$$\mathbb{E}_{(\bm x, \bm h, \bm y)\sim \mathcal{D}}[\langle s(\bm f,\bm x,\bm h,\bm y, \mathcal{D}),g(f(\bm x), \bm x)\rangle]=\mathbb{E}_{(\bm x,y)\sim \mathcal{D}}[s(f(\bm x),y)c(f(\bm x),\bm x)].$$
$$\mathbb{E}_{(\bm x, \bm h, \bm y)\sim \mathcal{D}}[\langle s(\bm f, \bm x,\bm h,\bm y, \mathcal{D}),g(f(\bm x), \bm x)\rangle]\leq \alpha, \forall g\in \mathcal{G}. $$
$$\Leftrightarrow \quad \mathbb{E}_{(\bm x,y)\sim \mathcal{D}}[s(f(\bm x),y)c(f(\bm x),\bm x)]\leq \alpha, \forall c \in \mathcal{C}.$$
Thus, the formulation of $(\bm s,\mathcal{G}, \alpha)$-GMC is reduced to the $s$-HappyMap.
\subsubsection{Outcome Indistinguishability}
Recall the definition of the outcome indistinguishability is
$$\mathbb{E}_{(\bm x, o_{\bm x}^*)\sim \mathcal{D}}[A(\bm x,\tilde{o}_{\bm x},\tilde{\bm p})-A(\bm x,o_x^*,\tilde{\bm p})]\leq \epsilon,\quad \forall A\in \mathcal{A}$$ 
where $\mathcal{A}$ is the set of discriminators, $\tilde{o}_{\bm x}$ is the output distribution to learn and $o_{\bm x}^*$ is the true underlying distribution. The goal is to find $\tilde{o}_{\bm x}$ such that all discriminators fail to identify it from the true $o_{\bm x}^*$.

And $(\bm s,\mathcal{G}, \alpha)-$GMC (where $\bm s(\bm x) \in \mathbb{R}^m$) is defined as:
$$\mathbb{E}_{(\bm x,\bm h,\bm y)\sim \mathcal{D}}[\langle \bm g(\bm x, \bm f(\bm x))\bm s(\bm f, \bm x, \bm h, \bm y, \mathcal{D})\rangle]\leq \alpha,\quad \forall \bm g\in \mathcal{G}.$$

We transform $(\bm s, \mathcal{G}, \alpha)$-GMC into Outcome Indistinguishability by defining the following function: (The left side of the equation represents the notation of $(\bm s, \mathcal{G}, \alpha)$-GMC, while the right side represents the notation of Outcome Indistinguishability.)

$O = \Delta O = \{x\in [0,1]^m: \sum_{i=1}^m x_i=1\},\quad \bm f(\bm x)=\tilde{\bm p}(\bm x)\in \mathbbm{R}^m$ is the probability distribution of over the output space. Denote $\mathcal{Y} = \{\mathcal{Y}_1, \mathcal{Y}_2, ...,\mathcal{Y}_K\}$ and $\bm p_{y}^* = (\Prob[y = \mathcal{Y}_1],\Prob[y = \mathcal{Y}_2],...,\Prob[y = \mathcal{Y}_K]),$ we further set
$$\bm s(\tilde{\bm p},\bm x, h, y, \mathcal{D})=\bm{\tilde{p}}(\bm x)-\bm p^*_{y}\in \mathbbm{R}^m.\quad \mathcal{G} = \{\bm g(\tilde{\bm p}(\bm x), \bm x)=A(\bm x,\cdot ,\bm{\tilde{p}})\in\mathbbm{R}^m: A\in \mathcal{A}\}.$$
$$\mathbb{E}_{(\bm x, \bm h, y)\sim \mathcal{D}}[\langle \bm s(\bm f,\bm x,h,y,\mathcal{D}),\bm g(\bm f(\bm x), \bm x)\rangle]=\mathbb{E}[A(\bm x,\tilde{o}_{\bm x},\bm{\tilde{p}})-A(\bm x,o_x^*,\bm{\tilde{p}})].$$
$$\mathbb{E}_{(\bm x, \bm h, y)\sim \mathcal{D}}[\langle \bm s(\bm f,\bm x,h,y,\mathcal{D}),\bm g(\bm f(\bm x), \bm x)\rangle]\leq \alpha, \forall \bm g\in \mathcal{G}. \quad \Leftrightarrow \quad \mathbb{E}[A(\bm x,\tilde{o}_{\bm x},\bm{\tilde{p}})-A(\bm x,o_x^*,\bm{\tilde{p}})]\leq \alpha, \forall A \in \mathcal{A}.$$
Thus, the formulation of $(\bm s,\mathcal{G}, \alpha)$-GMC is reduced to the Outcome Indistinguishability. 
\subsubsection{Multivalid Prediction}
Recall the definition of the $(\theta, m)-$multivalid prediction is 
$$
\left|\bar{H}\left(G^{(T)}(i)\right)-(1-\delta)\right| \leq \theta \left(\left|G^{(T)}(i)\right|\right).\quad \forall G\in \mathcal{G}, \forall i\in[m].
$$
where $$
G^{(t)}(i)=\left\{\tau \in[t]: x_\tau \in G, q^{(t)} \in [\frac{i-1}{m},\frac{i}{m})\right\},$$
$$\bar{H}(S)=\frac{1}{|S|} \sum_{t \in S} \mathbbm{1}_{\{s^{(t)} \leq q^{(t)}\}}.$$
Recall that the finite-sampling version $(\bm s,\mathcal{G}, \alpha)$-GMC (where $\bm s(\bm x) \in \mathbb{R}^m$) is defined as:
$$\frac{1}{T}\sum_{i=1}^{T}[\langle \bm g(\bm x_{(i)},\bm f(\bm x_{(i)})),\bm s(\bm f, \bm x_{(i)}, \bm h, \bm y_{(i)}, \mathcal{D})\rangle]\leq \alpha,\quad \forall \bm g\in \mathcal{G}.$$
Where $T$ denotes the sample size, $\{(\bm x_{(i)},\bm y_{(i)})\}$ denotes the data.
We transform the finite sampling version of $(\bm s, \mathcal{G}, \alpha)$-GMC into $(\theta, m)-$multivalid by defining the following function: (The left side of the equation represents the notation of $(\bm s, \mathcal{G}, \alpha)$-GMC, while the right side represents the notation of $(\theta, m)-$multivalid.)

Let $m=1, h=s, f=q, s(f, \bm x,h,y, \mathcal{D}) = \mathbbm{1}_{\{s \leq q\}}-(1-\delta), \mathcal{G}= \{\mathbbm{1}_{\{x\in G\}}\mathbbm{1}_{q\in [\frac{i-1}{m},\frac{i}{m})}:G\in \mathcal{G}, i\in [m]\}, \theta(|G^{(t)}(i)|)=\frac{\alpha T}{|G^{(t)}(i)|}$.

When $g(f(\bm x),\bm x)=\mathbbm{1}_{\{x\in G\}}\mathbbm{1}_{\{q\in [\frac{i-1}{m},\frac{i}{m})\}}$,
\begin{align}
    \frac{1}{T}\sum_{i=1}^{(T)}[\langle \bm g(\bm x_{(i)}, \bm f(\bm x_{(i)}))\bm s(\bm f, \bm x_{(i)}, \bm h, \bm y_{(i)}, \mathcal{D})\rangle]&=\frac{1}{T}\sum_{t=1}^Tg(q(\bm x^{(t)}), \bm x^{(t)})(\mathbbm{1}_{\{s(\bm x^{(t)},y^{(t)}) \leq q(\bm x^{(t)})\}}-(1-\delta))\nonumber\\
    &=\frac{|G^{(T)}(i)|}{T}|\bar{H}(G^{(T)}(i))-(1-\delta)|\nonumber\\
    &\leq \alpha\nonumber\\
    \Leftrightarrow |\bar{H}\left(G^{(T)}(i)\right)-(1-\delta)| &\leq \frac{\alpha T}{|G^{(T)}(i)|}=\theta(|G^{(T)}(i)|)).\nonumber
\end{align}
So
$$
\frac{1}{T}\sum_{i=1}^{T}[\langle \bm g(\bm x_{(i)}, \bm f(\bm x_{(i)}))\bm s(\bm f, \bm x_{(i)}, \bm h, \bm y_{(i)}, \mathcal{D})\rangle]\leq \alpha, \quad \forall \bm g\in \mathcal{G}
$$
$$\Leftrightarrow \quad |\bar{H}\left(G^{(T)}(i)\right)-(1-\delta)| \leq \frac{\alpha T}{|G^{(T)}(i)|}=\theta(|G^{(T)}(i)|)), \forall G\in \mathcal{G}, \forall i\in[m].$$
Thus, the sampling version of $(\bm s, \mathcal{G}, \alpha)-$GMC is reduced to the MultiValid Prediction where $\theta$ is of a specific form. 
\subsubsection{Quantile Multicalibration}
The quantile calibration error of $q$-quantile predictor $f: \mathcal{X} \rightarrow[0,1]$ on group $g'$ is:
$$
Q(f, g')=\sum_{v \in R(f)} \Prob_{(\bm x, s) \sim \mathcal{S}}(f(\bm x)=v \mid g'(\bm x)=1)\left(q-\Prob_{(\bm x, s) \sim \mathcal{S}}(s \leq f(\bm x) \mid f(\bm x)=v, g'(\bm x)=1)\right)^2 .
$$
Recall that $f$ is $\alpha$-approximately $q$-quantile multicalibrated with respect to group collection $\mathcal{G'}$ if
$$
Q(f, g') \leq \frac{\alpha}{\Prob_{(\bm x, s) \sim \mathcal{S}}(g'(\bm x)=1)} \quad \text { for every } g' \in \mathcal{G'}.
$$
where $g'(\bm x)=1$ denotes that $x$ belong to a certain subgroup.
This is equal to
$$q^2\mathbb{E}_{(\bm x, s) \sim \mathcal{S}}[\mathbbm{1}_{\{g'(\bm x)=1\}}]-2q\mathbb{E}_{(\bm x, s) \sim \mathcal{S}}[\mathbbm{1}_{\{s\leq f\}}\mathbbm{1}_{\{g'(\bm x)=1\}}]+\sum_{v\in R(f)}\frac{(\mathbb{E}_{(\bm x, s) \sim \mathcal{S}}[\mathbbm{1}_{\{s\leq f\}}\mathbbm{1}_{\{g'(\bm x)=1\}}\mathbbm{1}^2_{\{f(\bm x)=v\}})]}{\mathbb{E}_{(\bm x, s) \sim \mathcal{S}}[\mathbbm{1}_{\{f(\bm x)=v\}}\mathbbm{1}_{\{g'(\bm x)=1\}}]}\leq \alpha,$$
$$\forall g'\in \mathcal{G'}.$$
We set $\mathcal{G}=\{\pm \mathbbm{1}_{\{g'(\bm x)=1\}}:g'\in \mathcal{G'}\}, s(f,\bm x,h,y,\mathcal{D})=q-\mathbbm{1}_{\{s\leq f\}}$, then $\mathbb{E}_{(\bm x,h,y)\sim \mathcal{D}} = \mathbb{E}_{(\bm x, s) \sim \mathcal{S}}[g(\bm x)(q-\mathbbm{1}_{\{s\leq f\}})].$ (The left side of the equation represents the notation of $(\bm s, \mathcal{G}, \alpha)$-GMC, while the right side represents the notation of $\alpha$-approximately $q$-quantile.) We have that
$$\lvert\mathbb{E}_{(\bm x, s) \sim \mathcal{S}}[\mathbbm{1}_{\{g'(\bm x)=1\}}(q-\mathbbm{1}_{\{s\leq f\}})]\rvert \leq \alpha \quad \forall g'\in \mathcal{G'}.$$
This is equal to
$$\lvert q\mathbb{E}_{(\bm x, s) \sim \mathcal{S}}[\mathbbm{1}_{\{g'(\bm x)=1\}}] - \mathbb{E}_{(\bm x, s) \sim \mathcal{S}}[\mathbbm{1}_{\{g'(\bm x)=1\}}\mathbbm{1}_{\{s\leq f\}}]\rvert \leq \alpha \quad \forall g'\in \mathcal{G'}.$$
The two definitions can't be reduced to each other because of the presence of the term \\
$\sum_{v\in R(f)}\frac{(\mathbb{E}_{(\bm x, s) \sim \mathcal{S}}[\mathbbm{1}_{\{s\leq f\}}\mathbbm{1}_{\{g'(\bm x)=1\}}\mathbbm{1}^2_{\{f(\bm x)=v\}})]}{\mathbb{E}_{(\bm x, s) \sim \mathcal{S}}[\mathbbm{1}_{\{f(\bm x)=v\}}\mathbbm{1}_{\{g'(\bm x)=1\}}]}.$ Despite their differences, they both provide robust measures for assessing the extent of coverage in the given context.
\end{document}